\documentclass{article}
\usepackage[a4paper, total={6in, 8in}]{geometry}

\usepackage{graphicx}
\usepackage{verbatim}
\usepackage{latexsym}
\usepackage{setspace}
\usepackage{fullpage}

\usepackage{hyperref}

\usepackage{natbib}
\usepackage{url}

\usepackage{algorithm}
\usepackage{algpseudocode}

\usepackage{amsmath}
\usepackage{amsthm}
\usepackage{amsfonts}
\usepackage{amssymb}
\usepackage{nccmath}
\usepackage{mathrsfs}
\usepackage[bb=boondox]{mathalfa}
\usepackage{enumerate}

\usepackage[short]{optidef}

\usepackage{xcolor}

\usepackage{makecell}
\usepackage{multirow}
\usepackage{subcaption}

\usepackage{appendix}

\usepackage{bibentry}
\nobibliography*

\newcommand{\NN}{{\sf I\kern-0.14emN}}   
\newcommand{\ZZ}{{\sf Z\kern-0.45emZ}}   
\newcommand{\QQQ}{{\sf C\kern-0.48emQ}}   
\newcommand{\RR}{{\sf I\kern-0.14emR}}   

\newcommand{\syncc}{~\stackrel{\textstyle \rhd\kern-0.57em\lhd}{\scriptstyle L}~}

\newtheorem{definition}{Definition}[section]
\newtheorem{theorem}{Theorem}[section]
\newtheorem{assumption}{Assumption}[section]
\newtheorem{lemma}{Lemma}[section]
\newtheorem{example}{Example}[section]
\newtheorem{remark}{Remark}[section]
\newtheorem{corollary}{Corollary}[section]
\newtheorem{claim}{Claim}[section]

\DeclareMathOperator*{\argmax}{argmax}

\DeclareMathOperator*{\diam}{diam}
\DeclareMathOperator{\regret}{Regret}

\DeclareMathOperator*{\full}{full}

\title{Representative Action Selection for Large Action Space: From Bandits to MDPs}

\author{
Quan Zhou\thanks{Work done while at Technion. Corresponding author: \texttt{quan.zhou@campus.technion.ac.il}} \\
Technion \\
\and
Shie Mannor \\
Technion
}

\date{}

\begin{document}

\maketitle

\begin{abstract}
We study the problem of selecting a small, representative action subset from an extremely large action space shared across a family of reinforcement learning (RL) environments---a fundamental challenge in applications like inventory management and recommendation systems, where direct learning over the entire space is intractable.
Our goal is to identify a fixed subset of actions that, for every environment in the family, contains a near-optimal action, thereby enabling efficient learning without exhaustively evaluating all actions.

This work extends our prior results for meta-bandits to the more general setting of Markov Decision Processes (MDPs). We prove that our existing algorithm achieves performance comparable to using the full action space. This theoretical guarantee is established under a relaxed, non-centered sub-Gaussian process model, which accommodates greater environmental heterogeneity. Consequently, our approach provides a computationally and sample-efficient solution for large-scale combinatorial decision-making under uncertainty.
\end{abstract}

\section{Introduction}

Reinforcement Learning (RL)—spanning problems from multi-armed bandits (MAB) \citep{lai1985asymptotically} to Markov decision processes (MDP)—is fundamentally concerned with learning optimal actions through interaction with an environment to maximize long-term return. We study a meta-framework for RL defined by a family of environments that share a common, but extremely large action space. The high cardinality of this shared action space renders direct learning computationally intractable or sample-inefficient.
Our objective is to determine whether a small set of representative actions can be algorithmically identified such that its performance is comparable to the full action space across the entire family.
Prior work in RL typically seeks a single optimal action or a policy for a fixed environment. In contrast, our goal is to select a subset of actions that is likely to contain a near-optimal action across an entire family of environments. This constitutes a large-scale stochastic combinatorial problem, which arises naturally in applications like inventory management, robotics, and online recommendation systems, where decisions must be made from a vast space of possibilities that cannot be exhaustively evaluated.

To build intuition, consider a pharmacy that must decide which medicines to stock. The space of possible drugs (actions) is vast, and each patient (environment) has a unique response profile. In the bandit setting, a patient arrives needing a single, effective prescription. The pharmacist's goal is to identify the drug in stock that will have the best effect on the patient (i.e., yield the highest reward), without prior knowledge of each drug's efficacy for that specific individual. In the MDP setting, the pharmacist must manage treatments for patients whose health states evolve over time. 
Crucially, the problem exhibits structure: drugs with similar compositions have similar effects, and patients with similar health profiles respond to treatments similarly. We capture these correlations via a Gaussian process.
This structure makes subset selection both natural and necessary; instead of stocking every drug, the pharmacy can maintain a small, representative set. The optimal subset must cover diverse therapeutic needs—stocking both a painkiller and an antibiotic—while also adapting to demand, prioritizing common treatments like flu medicine during winter and deprioritizing drugs for rare conditions.



In our prior work \cite{zhou2025representative}, we introduced a representative action subset selection framework for meta-bandits, designed for a family of bandit environments. 
In this formulation, if a decision-maker plays an action \(a \in \mathcal{A}_{\full}\), they observe a stochastic reward. The expected reward is defined as the inner product \(Q^*(a):=\langle \phi(a), \theta \rangle\), where \(\phi: \mathcal{A}_{\full} \to \mathbb{R}^n\) is a fixed but unknown feature mapping and the environment parameter \(\theta \in \mathbb{R}^n\) is drawn from an unknown multivariate Gaussian distribution.
Consequently, the collection of random variables \(\{Q^*(a)\}_{a \in \mathcal{A}_{\full}}\) forms a Gaussian process (GP) \citep{vershynin2018high}.
We then proposed a simple algorithm, inspired by \cite{haussler1986epsilon}, to construct a measure-theoretic \(\epsilon\)-net over the full action space. This algorithm outputs a small set of representative actions whose performance is comparable to that of universal coverage, while avoiding both the combinatorial explosion of evaluating all possible subsets and the curse of dimensionality associated with forming a universal coverage directly. The resulting computational efficiency enables application to extremely large action spaces.

This paper extends our framework from bandit environments to the more general setting of MDPs.
This extension introduces an additional state space $\mathcal{X}$ and allows the fixed feature mapping $\phi$ to be state-dependent.
In a given MDP, we work with the optimal action-value function, $Q^*(x,a)$. This is a standard concept in RL that quantifies the maximum long-term return achievable by taking action $a$ in state $x$ and acting optimally thereafter. We model this function as:
\[Q^*(x,a):=f\Big(\phi(x,a),\theta\Big),\;\textnormal{where  }\phi:\mathcal{X}\times\mathcal{A}_{\full}\to\mathbb{R}^n.\]
Here, $f$ is a fixed but unknown function, and the MDP environment parameter \(\theta \in \mathbb{R}^n\) is drawn from a distribution. 
A notable special case of this framework is the linear MDP \citep{jin2020provably}, where $f(\theta,\phi):=\langle\theta,\phi\rangle$.  
We assume that the collection \(\{Q^*(a)\}_{a \in\mathcal{X}\times\mathcal{A}*{\full}}\) forms a (possibly non-centered) sub-Gaussian process. This assumption accommodates significant heterogeneity across the family of MDPs, allowing for variations in state transition dynamics, reward functions, and initial state distributions.
Our results demonstrate that the algorithm retains its coverage and efficiency guarantees despite the added complexity of the state space and the more general sub-Gaussian structure.

\paragraph{Contributions}
We extend the work of \cite{zhou2025representative} in the following key aspects:

\begin{itemize}
\item \textbf{Extended Framework to MDPs:} We generalize the previous meta-bandit framework (a family of linear bandits) to a meta-MDP framework (a family of linear MDPs). The meta-bandit setting is recovered as a special case. Furthermore, we show that the linear structure is not strictly necessary; our results hold as long as each state-action pair has a fixed feature representation across the family of environments.

\item \textbf{Relaxed Theoretical Assumptions:}
Our theoretical analysis operates under significantly weaker assumptions.

\begin{itemize}
\item \textbf{Non-centered Sub-Gaussian Process:} Prior work modeled the expected rewards as a centered Gaussian process. We relax this to a more general non-centered sub-Gaussian process, accommodating a wider range of environment distributions.

\item \textbf{Near-Optimal Policies:} We relax the algorithmic requirement of the meta-bandit framework, which necessitated finding the exact optimal action for every environment. In this work, we provide a regret bound that accommodates the use of near-optimal policies.

\end{itemize}

\item \textbf{Analysis of Large State Spaces:} A challenge in moving from bandits to MDPs is the computational complexity of large state spaces. We analyze how approximation errors from state abstraction propagate into the final regret bound. (We note that this consideration does not apply to the meta-bandit special case.)
\end{itemize}

\paragraph{Organization.}
Section~\ref{sec:meta-bandit} reviews the meta-bandit framework and presents a general algorithm that is simple to state and applies to both bandit and MDP environments.
Section~\ref{sec:preliminaries} introduces the preliminaries for general parametric MDPs and establishes the necessary notation.
Section~\ref{sec:meta-mdp} extends our framework to MDPs, outlining key assumptions and formally defining the learning objective.
Section~\ref{sec:reference} provides theoretical results for a reference action set, which serve as preliminary tools for analyzing our main algorithm.
Section~\ref{sec:theory} presents the main theoretical guarantees for the algorithm, including an analysis of performance in large state spaces.
Section~\ref{sec:examples} illustrates the framework with several concrete examples.

\paragraph{Notation}
Let \(n\) be a natural number, possibly infinite. We write \(\langle \cdot, \cdot \rangle\) for the Euclidean inner product, \(\|\cdot\|_2\) for the Euclidean norm, and \(\|A\|\) for the operator norm of a matrix \(A\). 
For a compact set \(\mathcal{S}\subset \mathbb{R}^n\), \(\diam(\mathcal{S}) := \max_{a,b\in\mathcal{S}} \|a-b\|_2\), \(|\mathcal{S}|\) denotes its cardinality, and \(\Delta(\mathcal{S})\) the set of probability distributions over \(\mathcal{S}\). 
For sets \(\mathcal{S}_1,\mathcal{S}_2\subseteq \mathbb{R}^n\), the Minkowski sum and difference are
\[
\mathcal{S}_1 \oplus \mathcal{S}_2 := \{ x_1 + x_2 : x_1 \in \mathcal{S}_1, x_2 \in \mathcal{S}_2 \}, \quad
\mathcal{S}_1 \ominus \mathcal{S}_2 := \{ x_1-x_2 : x_1 \in \mathcal{S}_1, x_2 \in \mathcal{S}_2 \}.
\]

\subsection{Related Work}

\paragraph{Meta-Learning.}
Meta-learning, or multi-task learning \citep{caruana1997multitask}, studies how agents can quickly adapt to new tasks drawn from a distribution by leveraging experience from previous tasks. 
The theoretical study of meta-learning was initiated by \citet{baxter2000model}, who provided the first formal statistical guarantees for learning across multiple tasks. They modeled the relationship between tasks by assuming the existence of an environment of tasks and an (unknown) environment distribution from which both observed training tasks and future tasks are i.i.d. samples.
Subsequent works developed theoretical guarantees using various techniques, including Rademacher averages \citep{ando2005framework}, analysis of special cases where input data lie in a linear (potentially infinite-dimensional) space with a shared linear preprocessor \citep{maurer2006bounds}, excess risk bounds \citep{maurer2016benefit}, PAC-Bayesian analysis \citep{zakerinia2024more,pentina2014pac}, and information-theoretic methods \citep{chen2021generalization}.

Prominent algorithmic approaches include gradient-based adaptation methods such as MAML \citep{finn2017model}. While most works have demonstrated empirical success (see the survey by \citet{vettoruzzo2024advances}), the majority focus on algorithmic performance and adaptation heuristics, for example, gradient-based methods for updating model parameters for each task. A major obstacle in these methods is the large-scale bi-level optimization problem inherent in popular meta-learning algorithms such as MAML. The associated computational and memory costs can be significant, motivating the development of more practical approaches, particularly for very large-scale problems.

\paragraph{Transfer and Multi-Task RL.}
Meta-RL is a special case of meta-learning where each task corresponds to a reinforcement learning (RL) environment \citep{sutton1998reinforcement,puterman2014markov}. This line of work focuses on exploiting shared structure among tasks to improve learning efficiency.

Several approaches have modeled the environment context as a latent variable (first proposed in \cite{hallak2015contextual}) to capture task-specific variations. Later work modeled the latent context probabilistically \citep{rakelly2019efficient}, enabling adaptation to new tasks based on the inferred context. In terms of shared structure, some works use Latent MDPs \citep{kwon2021rl} and Lipschitz Block Contextual MDPs \citep{sodhani2022block,sodhani2021multi} to analyze performance through task similarity metrics.

Algorithms in this line of work typically focus on inferring the context of a sampled task. In contrast, our paper does not focus on context inference: we provide a fixed subset of actions to use across the family of RL environments. Fast adaptation is enabled by exploring a smaller set of actions.

\paragraph{Policy Compression.}

We identify several distinct approaches to policy compression. The first, commonly known as \emph{policy distillation} \citep{rusu2015policy}, transfers knowledge from a complex teacher policy (or ensemble) into a simpler student policy. This technique aims to create a compact model that mimics the original policy's behavior while being more efficient to deploy—particularly valuable when the teacher is computationally expensive (e.g., a large neural network) and the student must operate under constraints like limited memory or real-time decision-making. Following \cite{ave2025temporal}, this method can be viewed as compressing policy storage requirements.

A second approach, \emph{temporal distillation} \citep{ave2025temporal}, addresses compression through time rather than model complexity. Instead of simplifying the policy network, it employs frame-skipping where actions are repeated across multiple timesteps—a standard practice in high-frame-rate environments.

The third direction, \emph{policy space compression} \citep{mutti2022reward}, aims to identify a minimal set of representative policies whose state-action distributions approximate any achievable distribution under the full action space. This method focuses on a single MDP environment through bi-level optimization. Related work by \cite{tenedini2025parameters} uses autoencoders to compress the parameter space forming the policy space across different tasks. We note that this line of research seeks universal coverage of the policy space, as measured by a divergence between the state-action distributions of any two policies.

\section{Preliminaries: Meta-Bandits}
\label{sec:meta-bandit}

We briefly recall the meta-bandits framework from \cite{zhou2025representative}, which generalizes the standard MAB setting \cite{lai1985asymptotically}. In one MAB, a decision-maker sequentially selects actions and observes stochastic rewards generated from initially unknown distributions. The central challenge is to maximize cumulative reward while balancing the fundamental trade-off between exploration---experimenting with actions to obtain information about their reward distributions---and exploitation---leveraging existing information to select actions with the highest estimated rewards.

Formally, a bandit environment involves an action space $\mathcal{A}_{\full}$, where playing an action $a$ yields a reward with unknown expectation $Q^*(a)$. In our setting, however, the decision-maker only has access to a restricted subset $\mathcal{A} \subset \mathcal{A}_{\full}$. For a given bandit environment parameterized by $\theta$, if the optimal action in $\mathcal{A}_{\full}$ lies within $\mathcal{A}$, the decision-maker benefits from a smaller set of actions to explore while still being able to identify the global optimum. Conversely, if the optimal action lies outside $\mathcal{A}$, the decision-maker inevitably incurs regret due to the unavailability of the global optimal action.
Thus, for a bandit environment $\theta$, we define the regret as the difference in expected reward between having access to the full action space versus being restricted to the subset:
\begin{equation}
\regret := \max_{a\in\mathcal{A}_{\full}} Q^*(a) - \max_{a'\in\mathcal{A}} \;Q^*(a'),
\label{equ:regret-define}
\end{equation}
which depends on the sampled bandit environment and is therefore a random variable. Our objective is to identify a small subset $\mathcal{A}$ that minimizes the expected regret $\mathbb{E}_{\theta}\left[\regret\right]$ over all possible bandit environments, making the underlying optimization both stochastic and combinatorial.

We first restrict our attention to the canonical Gaussian process \citep[Chapter~7]{vershynin2018high}:
\begin{equation}
Q^*(a) := \langle \phi(a), \theta \rangle, \quad \forall a\in\mathcal{A}_{\full},\;\textnormal{ where }\theta \sim \mathcal{N}(0, I).
\label{equ:mu-define}
\end{equation}
\(\phi: \mathcal{A}_{\full} \to \mathbb{R}^n\) is a fixed but unknown feature mapping, and $I$ is the $n$-dimensional identity matrix.
Then, to gain geometric intuition, consider the notion of \textbf{extreme points}. We define the extreme points as those $\phi(a)\in \mathcal{A}_{\full}$ for which there do not exist distinct $a', a^{\dag} \in \mathcal{A}_{\full}$ and $\lambda \in (0,1)$ such that $\phi(a) = \lambda \phi(a') + (1 - \lambda)\phi(a^{\dag})$. 
\begin{example}
By the extreme point theorem, if we select all extreme points---denoted $\mathcal{A} = \{a_1, \dots, a_K\}$---as representatives of the full action space, the regret defined in \eqref{equ:regret-define} is zero. This is because any $a \in \mathcal{A}_{\full}$ can be expressed as a convex combination of the extreme points: $\phi(a) = \lambda_1 \phi(a_1) + \dots + \lambda_K \phi(a_K)$, where $\lambda_i \geq 0$ and $\sum_{i=1}^K \lambda_i = 1$.

Thus, for any $\theta\in\mathbb{R}^n$:
\begin{equation*}
\begin{split}
\langle \phi(a),\theta \rangle
= \sum_{i=1}^K \lambda_i\langle \phi(a_i),\theta\rangle
\leq \sum_{i=1}^K \lambda_i \max_{a'\in\mathcal{A}}\langle \phi(a'),\theta\rangle = \max_{a'\in\mathcal{A}}Q^*(a').
\end{split}
\end{equation*}
By \eqref{equ:mu-define}, it yields
\[\max_{a'\in\mathcal{A}} Q^*(a')\leq\max_{a\in\mathcal{A}_{\full}}Q^*(a)\leq \max_{a'\in\mathcal{A}} Q^*(a'),\]
where the left inequality uses $\mathcal{A}\subseteq\mathcal{A}_{\full}$. Thus, the two quantities $\max_{a\in\mathcal{A}} Q^*(a)$ and $\max_{a\in\mathcal{A}_{\full}}Q^*(a)$ are equal.
\end{example}


This example highlights how a geometric approach can be used to solve the stochastic combination problem.
However, even if one only needs the extreme points, the set of extreme points may still be large, e.g, the extreme points of a Euclidean ball are infinite. To address this, we next introduce the notion of \textbf{$\epsilon$-nets}.
Without loss of generality, we assume \(\mathcal{A}_{\full}\) consists only of the extreme points of \(\mathcal{A}_{\full}\), as they are the only points of interest.

\subsection{Connection to Bayesian Bandit Regret}
This objective in Equation~\eqref{equ:mu-define} is motivated by the classic Bayesian regret in the bandit literature \citep{agrawal2012analysis} that typically scales with the number of available actions.
However, if a subset $\mathcal{A}$ is carefully chosen, the resulting Bayesian bandit regret can be significantly lower.
This is especially beneficial when the action space is large, as even the initialization phase can be computationally expensive.
To see this, we can decompose the Bayesian bandit regret as follows:
\begin{align*}
&\textnormal{Bayesian (Bandit) Regret}:=\mathbb{E}\sum_{t=1}^K \left[\max_{a \in \mathcal{A}_{\full}} Q^*(a)-  Q^*({a_t})\right]\\
&=\mathbb{E}\sum_{t=1}^K \left[\max_{a' \in \mathcal{A}} Q^*(a')\!-\! Q^*(a_t) +\max_{a \in \mathcal{A}_{\full}} Q^*(a)\!-\!\max_{a' \in \mathcal{A}} Q^*(a')\right]\\
&\leq  C\sqrt{|\mathcal{A}|\cdot N\log N} + N\cdot\mathbb{E}_{\theta}[\regret].
\end{align*}
The first equality follows from the definition of Bayesian bandit regret, which is the expected difference in cumulative outcomes over $N$ rounds between always selecting the optimal action and following a decision rule. Let $a_t \in \mathcal{A}$ denote the action chosen by the rule (e.g., Thompson Sampling \citep{agrawal2012analysis}) in round $t$. The expectation is taken over both the randomness in the bandit environments and the actions selected by the decision rule.
The inequality follows from the well-known regret bounds for Thompson Sampling \citep{lattimore2020bandit}\footnote{The dependence on action space cardinality could be weaker when feature vectors are known or incorporated into kernel functions. However, we present this general bound for the most general case where the linear structure in Equation~\ref{equ:mu-define} is not assumed.}, where $C > 0$ is a constant.
Note that if the decision rule has access to the full action space, the Bayesian bandit regret is instead bounded by $C\sqrt{|\mathcal{A}_{\full}| \cdot N \log N}$.


\subsection{Epsilon Nets}
\label{sec:epsilon-nets}

To identify a subset $\mathcal{A}$ that minimizes the expected regret $\mathbb{E}_{\theta}\left[\regret\right]$,  a natural approach is to construct a grid over the action space, where the grid points serve as representative actions. This ensures that for every action in the full space, there exists a representative that is close to it. This idea is formally captured by the notion of a (geometric) $\epsilon$-net.

To proceed, we clarify what we mean by an $\epsilon$-net, as there are at least two definitions: one from a geometric perspective \citep[Chapter~4]{vershynin2018high} and another from a measure-theoretic perspective \citep[Chapter~10]{matousek2013lectures}. 

\begin{itemize}
\item 
A subset \( \mathcal{A} \subseteq \mathcal{A}_{\full} \) is called a \textbf{Geometric \(\epsilon\)-net} if, for all \( a \in \mathcal{A}_{\full} \), there exists \( a' \in \mathcal{A} \) such that  
\[\|a-a'\|_2 < \epsilon. \]  

\item 
Let $\{r_{\ell}\}_{\ell\leq m}$ be a finite partition of the extreme points into disjoint clusters such that $\cup_{\ell\leq m}\; r_{\ell} = \mathcal{A}_{\full}$. 
Given a measure \( q \) assigning a value to each cluster.  
A subset \( \mathcal{A} \subseteq \mathcal{A}_{\full} \) is called a \textbf{Measure-Theoretic \(\epsilon\)-net} with respect to measure \( q \) if, for any cluster \( r_{\ell} \), we have:  
 \[ r_{\ell} \cap \mathcal{A} \neq \emptyset \quad \text{whenever} \quad q(r_{\ell}) > \epsilon. \]  

\end{itemize}


A geometric $\epsilon$-net ensures small regret because if two actions $a, a' \in \mathcal{A}_{\full}$ are close in the Euclidean sense, then the deviation between $Q^*(a)$ and $Q^*(a')$ is small in the $L^2$-sense (i.e., their expected squared difference is small):
\begin{equation}
\|Q^*(a) \!-\! Q^*(a')\|_{L^2} =\!\left( \mathbb{E}(a \!-\! a')^{\top}\theta \theta^{\top}(a \!-\! a')\right)^{1/2}\!=\|a \!-\! a'\|_2, 
\label{equ:L2norm}
\end{equation}
where $\theta^{\top}$ denotes the transpose of $\theta$. The equalities use \eqref{equ:mu-define} and $\mathbb{E}\theta\theta^{\top}=I$. 
Therefore, by definition, a geometric $\epsilon$-net guarantees the existence of an action $a \in \mathcal{A}$ whose expected outcome $Q^*(a)$ is close to that of the optimal action for any given bandit environment.
However, this net suffers from the curse of dimension: e.g., for \([0,1]^n\), the number of points needed to form a geometric \(\epsilon\)-net grows as \((1/\epsilon)^n\).

The measure-theoretic $\epsilon$-net addresses this issue. 
Put simply, the measure-theoretic $\epsilon$-net restricts the grid construction to only the most important clusters, as determined by the $q$-measure.
Selecting one point from each cluster forms a geometric $\epsilon$-net (which will later be formally defined as a reference subset), where $\epsilon$ is the maximum diameter of all the clusters. This construction yields bounded expected regret, as each selected point serves as a representative for its cluster, and all other points within the same cluster have bounded deviation from the representative.
Furthermore, clusters in $\mathcal{A}_{\full}$ that are deemed unimportant for expected regret can be ignored: even if no representative points are selected from these clusters, the resulting expected regret will not increase significantly compared to that of the reference subset.

\subsection{Epsilon Net Algorithm}
To construct a measure-theoretic $\epsilon$-net, we employ the classic $\epsilon$-net algorithm of \citet{haussler1986epsilon}, which remains the simplest and most broadly applicable approach. See Algorithm~\ref{alg:valuefunction} for details. While the algorithm is presented in full generality and refers to a state space, in the meta-bandits setting we consider here, we focus on the simplest case where the state space $\mathcal{X}$ is a singleton.

With high probability, this algorithm produces a measure-theoretic $\epsilon$-net under the measure $q$, defined by  
\[
q(r_\ell) := \Pr\!\left[a^*(\theta) \in r_{\ell} \right],
\]  
where $a^*(\theta)$ is the optimal action for the environment $\theta$, and we assume the optimal action is unique with probability one.  
To see this, note that Algorithm~\ref{alg:valuefunction} samples $K$ i.i.d.\ points from each cluster according to the measure $q$.  
Fix a cluster with $q(r_\ell) > \epsilon$. The probability of not sampling any point in $r_{\ell}$ after $K$ draws is $(1-q(r_\ell))^K \leq \exp(-K q(r_\ell))$, where the inequality follows from \cite{topsoe12007some}.  
Finally, observe that there can be at most $\lfloor 1/\epsilon \rfloor$ clusters satisfying $q(r_\ell) > \epsilon$.

\begin{algorithm}[htp!]
\caption{Epsilon Net Algorithm}
\begin{algorithmic}[1]
\State \textbf{Input: } State space $\mathcal{X}$, Sample size $K$, Environment distribution $\rho$.
\State $\mathcal{A}(x)\gets\emptyset$, for $x\in\mathcal{X}$.
\For{$1,\dots, K$} 
    \State Sample an RL environment $\theta\sim p$.
    \State Find the optimal value function $Q^*(x,a)$, for $x\in\mathcal{X}$, $a\in\mathcal{A}_{\full}(x)$.
    \State $\mathcal{A}(x)\gets \mathcal{A}(x)\cup\left\{\argmax_{a\in \mathcal{A}_{\full}(x)} Q^*(x,a)\right\}$, for $x\in\mathcal{X}$.
    \Comment{Repetitions are allowed}
\EndFor
\end{algorithmic}
\label{alg:valuefunction}
\end{algorithm}


\section{Preliminaries: MDPs}
\label{sec:preliminaries}

We now move from bandits to MDPs. The key difference between the two settings lies in the presence of state dynamics. A bandit problem can be viewed as an MDP with a single state, where each action yields a reward drawn from a fixed distribution and the next state is trivially the same. In contrast, an MDP involves a finite state space and a transition kernel that governs the evolution of the state. As a result, the decision-maker must account not only for the immediate rewards of actions but also for their long-term consequences through the induced state transitions.

A general MDP is characterized by the tuple $(\mathcal{X}, \mathcal{A}_{\full}, P, R, \gamma, x_0)$. The state space $\mathcal{X}$ is assumed to be a Borel space. The action space $\mathcal{A}_{\full}$ is state-dependent: for each $x \in \mathcal{X}$, the set of admissible actions is denoted by $\mathcal{A}_{\full}(x)$. The set of admissible state–action pairs is thus given by
\[
\mathcal{X}\times\mathcal{A}_{\full}:=\{(x,a)\mid x\in\mathcal{X}, a\in \mathcal{A}_{\full}(x)\}.
\]
$P(y \mid x,a)$ is a transition kernel on $\mathcal{X}$ given the state-action pair $(x,a)$.
$R(x,a)$ is the one-step reward function.
The parameter $\gamma\in[0,1)$ is the discount factor, which controls the relative importance of future rewards. A value of $\gamma=0$ corresponds to the case where only immediate rewards are considered.
The initial state $x_0$ is fixed at the start of the process and specifies the starting point for interactions with the MDP.
A decision rule is a mapping $\pi : \mathcal{X} \to \mathcal{A}$ such that $\pi(x) \in \mathcal{A}_{\full}(x)$ for all $x \in \mathcal{X}$. It specifies which action to take at each state.
We view $\pi$ as a stationary policy, that is, a Markov policy that uses the same decision rule $\pi$ at all time steps.

\subsection{Key Quantities}

It is often useful to quantify how good a given state or state–action pair is under a policy. The standard tools for this purpose are the \textbf{value functions}, defined as
\begin{align*}
V^{\pi}(x):=&\mathbb{E}\left[\sum_{t=0}^{\infty}\gamma^t R(x_t,a_t) \bigg| x_0=x\right],\\
Q^{\pi}(x,a):=& \mathbb{E}\left[\sum_{t=0}^{\infty}\gamma^t R(x_t,a_t) \bigg| x_0=x,a_0=a\right].
\end{align*}
In words, $V^{\pi}(x)$ denotes the expected discounted return when starting from state $x$ and following policy $\pi$, while $Q^{\pi}(x,a)$ denotes the expected discounted return when starting from state $x$, taking action $a$, and subsequently following $\pi$.
The \textbf{optimal value functions} are defined as 
\begin{align*}
V^*(x) :=& \max_{\pi} V^{\pi}(x),\\
Q^*(x,a) :=& \max_{\pi} Q^{\pi}(x,a),
\end{align*}
and they are related by
\begin{equation}
\begin{split}
V^*(x) =& \max_a Q^*(x,a),\\
Q^*(x,a) =& R(x,a) + \gamma \sum_{y\in\mathcal{X}} P(y \mid x,a) \, V^*(y).
\end{split}
\label{equ:Q-V-optimal}   
\end{equation}
The policy that attains the maximum is the optimal policy $\pi^*(x):=\arg \max_{\pi} Q^{\pi}(x,a)$.

Define the \textbf{(dis)advantage function} of $\pi$ with respect to the optimal policy $\pi^*$ as
\begin{equation}
 \begin{split}
 A^{*}(x,a) := Q^{*}(x,a) - V^{*}(x). 
\end{split}   
\label{equ:advantage-define}
\end{equation}

At each time step, there is some probability of being in each state.
Given an initial state $x_0$, define the discounted \textbf{state visitation distribution} induced by policy $\pi$:
\begin{align*}
d^{\pi,x_0}(x) := (1-\gamma) \sum_{t=0}^\infty \gamma^{t} \Pr(x_t = x \mid \pi, x_0=x_0),
\end{align*}
which aggregates these probabilities over all time steps, but with a discount factor applied so that states visited sooner count more.
The factor $(1-\gamma)$ is a normalization so that $d^{\pi,x_0}$ is a probability distribution over states.

\subsection{Structured RL Environments}

Although we allow for heterogeneity across the family of MDPs in their transition dynamics, reward functions, and initial state distributions, additional structure is required to make the problem tractable. To this end, we focus on structured environments, beginning with linear models and then extending to more general parametric forms.

\paragraph{Linear MDPs}

A linear MDP \citep{jin2020provably} with dimension \(n\) is a MDP for which there exists a fixed feature mapping
\[
\phi : \mathcal{X}\times \mathcal{A}_{\full} \to \mathbb{R}^n
\]
and unknown parameters \(\{\theta_P(y),y\in\mathcal{X};\theta_R\}\), such that

\begin{equation*}
P(y|x,a) = \langle \theta_P(y),\phi(x,a)\rangle\quad 
R(x,a) = \langle \theta_R,\phi(x,a)\rangle
\quad \forall y\in\mathcal{X},\; (x,a)\in \mathcal{X}\times\mathcal{A}_{\full}.
\end{equation*}
For this special class of MDPs, the state–action value functions admit a linear structure (see Proposition 2.3 of \cite{jin2020provably}):
\begin{equation}
\begin{split}
Q^*(x,a)=& R(x,a)+\gamma \sum_{y\in\mathcal{X}} P(y|x,a) V^*(y)\\
=& \langle \theta_R,\phi(x,a)\rangle+\gamma \sum_{y\in\mathcal{X}} \langle \theta_P(y),\phi(x,a)\rangle V^*(y)\\
=& \langle \theta,\phi(x,a)\rangle,
\end{split}
\label{equ:linear-Q-define}
\end{equation}
where $\theta:=\theta_R+\gamma\sum_{y\in\mathcal{X}} \theta_P(y) V^*(y)$. The first equality uses Equation~\eqref{equ:Q-V-optimal}. 
Note that when the characteristics $\{\theta_P(y), y \in X; \theta_R\}$ and the action space $\mathcal{A}_{\full}(x),x\in\mathcal{X}$ are fixed, the weight $\theta$ is also fixed.


\paragraph{Generalized Parametric MDPs}

To capture a broader family of problems, we consider generalized parametric MDPs, in which the optimal state-action value function takes the form
\[
Q^*(x,a) = f(\theta, \phi(x,a)),
\]
where $f : \Theta \times \mathbb{R}^n \to \mathbb{R}$ is a function and $\theta \in \Theta$ is the parameter of this MDP.
Thus, the linear MDP model corresponds to the special case $f(\theta,\phi)=\langle \theta,\phi\rangle$.
\textcolor{black}{This idea of using a general function class to approximate the Q functions has appeared in \cite{huang2021going}.}
This formulation enables us to study reinforcement learning in more expressive but still structured environments.

\section{Extended Framework to MDPs}
\label{sec:meta-mdp}

We consider the meta-MDPs framework, which generalizes the meta-bandits setting by introducing a family of MDPs. 
All MDPs share the same state-dependent action space \(\mathcal{A}_{\full}\).
Each MDP in the family is characterized by a fixed feature map \(\phi(x,a)\) and a parameter \(\theta\), with $\theta$ drawn from a distribution $p$. That is, \(\theta\) defines the rewards and transition dynamics of a particular MDP environment, while the feature map \(\phi\) is shared across all environments. 
Specifically, there exits a function $f : \Theta \times \mathbb{R}^n \to \mathbb{R}$ that
\begin{equation}
Q^*(x,a)=f(\theta, \phi(x,a)),\quad\forall (x,a)\in\mathcal{X}\times\mathcal{A}_{\full},\;\textnormal{ where }\theta\sim p.
\label{equ:mu-x-define}
\end{equation}
We make no assumptions on the initial state of each MDP; that is, the initial state $x_0$ can be an arbitrary state and may vary across the family of MDPs. 
Since \(\theta\) uniquely parameterizes each environment, we will refer to \(\theta\) simply as an \emph{environment} in the sequel, which may represent either a bandit environment or an MDP environment, as the meta-bandits framework is merely a special case. 
We also simply write \(\mathcal{A}\) to refer to the collection \(\{\mathcal{A}(x)\}_{x \in \mathcal{X}}\), analogous to the action set $\mathcal{A}$ in the meta-bandits framework.

In our setting, however, the decision-maker only has access to a restricted subset of actions \(\mathcal{A}(x) \subset \mathcal{A}_{\full}(x)\) for each state \(x \in \mathcal{X}\). For a given environment $\theta$, if the optimal action(s) lies within the available action set $\mathcal{A}$, then the decision-maker faces a smaller set of  suboptimal actions to explore and can quickly identify the optimal action at this state. In contrast, if the optimal action lies outside $\mathcal{A}$, the decision-maker incurs regret due to the unavailability of the globally optimal action.
Therefore, for \(\theta\), we define the state-dependent regret as
\begin{equation}
\regret(x):=\max_{a\in \mathcal{A}_{\full}(x)} Q^*(x,a)-\max_{b\in \mathcal{A}(x)} Q^*(x,b)\quad x\in\mathcal{X},\label{equ:regret-x-define}
\end{equation}
which is a random variable since \(\theta\) is drawn from \(p\), and it depends on the restricted action subset $\mathcal{A}$.
In this section, we aim to upper bound the term
\begin{equation}
\mathbb{E}_{\theta \sim p}\Big[ \max_{x\in\mathcal{X}} \regret(x) \Big].
\label{equ:obj-uniform}
\end{equation}
In particular, when the state space \(\mathcal{X}\) is a singleton, as in the meta-bandits framework, this term reduces to the objective \(\mathbb{E}_{\theta \sim p}[\regret]\) in that setting.
Also, since the term involves an expectation over the environment $\theta$, it now depends on the choice of action subset $\mathcal{A}$. 

We consider two types of action subsets $\mathcal{A}$.
The first type is a uniform coverage of the full action space. In the meta-bandits framework, this corresponds to a geometric epsilon-net, as discussed in Section~\ref{sec:epsilon-nets}.
While this subset would be the ideal choice in principle, finding such a coverage suffers from the curse of dimensionality. This curse appears in two ways: (i) the number of actions required to form a universal coverage can grow exponentially with the dimension $n$, and (ii) explicitly constructing such a net is a combinatorial problem.
The second type of subset is the output of Algorithm~\ref{alg:valuefunction}. We claim that the performance of this subset is, in expectation, close to that of a uniform coverage of the full action space. More importantly, the algorithm is simple to run in practice.

\subsection{Connection to Expected Performance Difference}

The objective defined in Equation~\eqref{equ:obj-uniform} is motivated by its direct connection to performance differences induced by restricting the available actions. 
This performance difference is typically quantified by the gap between the value functions of two policies at a given initial state \citep{kakade2002approximately}.
To make this connection precise, we recall the well-known performance difference lemma:

\begin{lemma}[Performance Difference Lemma \citep{kakade2002approximately}]\label{lem:performance-difference}
Given an initial state $x_0$ and a policy $\pi$, the difference in value functions from the optimal policy $\pi^*$ satisfies
\[
V^{*}(x_0) - V^{\pi}(x_0) = \frac{1}{1-\gamma} \, \mathbb{E}_{x \sim d^{\pi,x_0}} \big[ -A^{*}(x,\pi(x)) \big].
\]

\end{lemma}


We measure performance loss by comparing the optimal policy under the full action space with the optimal policy under the restricted action space. This quantifies the intrinsic cost of limiting the available actions, isolating the effect of the restriction from other sources of suboptimality.

\begin{lemma}\label{lem:performance-difference-linear}
For any initial state $x_0$, let $\pi$ denote the optimal policy within the restricted action subset $\mathcal{A}(x)$, $x\in\mathcal{X}$. The difference in value functions from the optimal policy $\pi^*$ satisfies is then bounded by
\begin{equation*}
V^{*}(x_0)-V^{\pi}(x_0)\leq \frac{1}{1-\gamma} \max_{x\in\mathcal{X}}\left\{\max_{a\in \mathcal{A}_{\full}(x)} Q^*(x,a)-\max_{b\in \mathcal{A}(x)} Q^*(x,b) \right\} ,\quad\forall x_0\in\mathcal{X}.
\end{equation*}

\begin{proof}
If for each state $x\in\mathcal{X}$, the decision-maker is restricted to a subset $\mathcal{A}(x)\subseteq \mathcal{A}_{\full}(x)$, we define a (possibly) suboptimal policy $\pi$:
\begin{equation}
\pi(x):=\argmax_{a\in \mathcal{A}(x)} Q^*(x,a) 
\quad \forall x\in\mathcal{X}.
\label{equ:localoptimal-policy}
\end{equation}

Therefore,
\begin{align}
A^*(x,\pi(x))= Q^*(x,\pi(x)) - V^*(x)= \max_{a\in \mathcal{A}(x)} Q^*(x,a) - \max_{b\in \mathcal{A}_{\full}(x)} Q^*(x,b),
\label{equ:A-decompose}
\end{align}
where the first equality follows from the definition of the (dis)advantage function in Equation~\eqref{equ:advantage-define}. The second equality follows from the definition of the (possibly) suboptimal policy $\pi$ in Equation~\eqref{equ:localoptimal-policy} and the relationship between the optimal value and Q-functions in Equation~\eqref{equ:Q-V-optimal}.

Then, by Lemma~\ref{lem:performance-difference}, we have
\begin{align*}
V^{*}(x_0)-V^{\pi}(x_0) 
=& \frac{1}{1-\gamma} \, \mathbb{E}_{\substack{x \sim d^{\pi,x_0}}} \big[ -A^{*}(x,\pi(x)) \big]\\
=& \frac{1}{1-\gamma} \,  \mathbb{E}_{\substack{x \sim d^{\pi,x_0}}} \big[\max_{a\in \mathcal{A}_{\full}(x)} Q^*(x,a)-\max_{b\in \mathcal{A}(x)} Q^*(x,b) \big]\\
\leq& \frac{1}{1-\gamma} \max_{x\in\mathcal{X}}\left\{\max_{a\in \mathcal{A}_{\full}(x)} Q^*(x,a)-\max_{b\in \mathcal{A}(x)} Q^*(x,b) \right\} \quad\forall x_0\in\mathcal{X},
\end{align*}
where the first equality uses Lemma~\ref{lem:performance-difference}, the second uses Equation~\eqref{equ:A-decompose}.
The inequality uses that not matter what is the policy $\pi$ and initial state $x_0$, the state visitation distribution $d^{\pi,x_0}$ is a probability distribution, such that $d^{\pi,x_0}\in\Delta(\mathcal{X})$.
Since the policy $\pi$ is within the restricted action subset, we complete the proof.
\end{proof}
\end{lemma}

Lemma~\ref{lem:performance-difference-linear} upper bounds the performance difference between the optimal policy under the full action space (which represents the best achievable performance if no restrictions are imposed) and the optimal policy under the restricted action space (which represents the best achievable performance given the constraint) for a single MDP environment.

In the meta-bandits setting, we consider a family of MDPs, where each environment $\theta$ is sampled from the distribution $p$. The expected performance difference over this family of MDPs, incurred by restricting the available actions to state-dependent subsets $\mathcal{A}(x)\subseteq \mathcal{A}_{\full}(x)$ for $x\in\mathcal{X}$, is then bounded by the maximal single-state regret.
Specifically, for any initial state \(x_0 \in \mathcal{X}\), we have
\begin{equation}
\textnormal{Expected Performance Difference}:=\mathbb{E}_{\theta\sim p}\Big[V^{*}(x_0) - V^{\pi}(x_0)\Big]
\le \frac{1}{1-\gamma} \mathbb{E}_{\theta \sim p}\Big[ \max_{x\in\mathcal{X}} \regret(x) \Big].\label{equ:performance-bound-connect}
\end{equation}
Therefore, by analyzing the expected maximal single-state regret across the family of MDPs, we can quantify the impact of action restrictions on overall decision-making performance.

\subsection{Assumptions \& Definitions}

We introduce the notions of \textbf{partition}, \textbf{reference subsets}, and \textbf{cluster regret}, which serve as the key technical tools for comparing the performance of the subset generated by Algorithm~\ref{alg:valuefunction} with that of a uniform coverage of the action space. Partitioning the action space into clusters enables us to precisely characterize the regions potentially “missed” by Algorithm~\ref{alg:valuefunction} and to quantify the corresponding performance loss at the cluster level. These notions form the basis for establishing a rigorous connection between the two types of action subsets and for deriving regret bounds in terms of cluster-level quantities.

\subsubsection{Finite Partition \& Reference Set}
We start by defining a finite partition of the environment space and then the action space, which will serve as the basic structure for the notions introduced later.
\begin{assumption}\label{ass:unique}
For each state $x\in\mathcal{X}$, the optimal action $\argmax_{a \in \mathcal{A}_{\full}(x)} Q^*(x,a)$ is unique with probability 1 with respect to the distribution of environments $p$.
\end{assumption}

The environment $\theta$ is sampled from a distribution $p$ whose support is denoted $\Theta\subseteq\mathbb{R}^n$. Let $\{\Theta_{\ell}\}_{\ell\leq m}$ be a finite partition of $\Theta$ into disjoint clusters. We assume that this partition of the environment space remains fixed throughout the paper.


Define the state-dependent partitions of action spaces based on this partition:
\begin{equation}
r_{x,\ell}:=\left\{a\in \mathcal{A}_{\full}(x)\mid \exists \theta\in \Theta_{\ell}: a\in\argmax_{a \in \mathcal{A}_{\full}(x)} Q^*(x,a)\right\}.\label{equ:cluster-define}
\end{equation}

\begin{assumption}\label{ass:compact}
For each $x\in\mathcal{X}$, $\ell\leq m$, the cluster $r_{x,\ell}$ is compact.
\end{assumption}

Let $\epsilon:=\max_{x\in\mathcal{X},\ell\leq m}\diam(r_{x,\ell})$.

We also define
\begin{equation}
p_{\ell}:=\Pr\left[\theta\in \Theta_{\ell} \right]=\Pr\left[\argmax_{a \in \mathcal{A}_{\full}(x)} Q^*(x,a)\in r_{x,\ell} \right], \quad\forall x\in\mathcal{X}.
\label{equ:q_l}
\end{equation}
Note that $\{p_{\ell}\}_{\ell\leq m}$ forms a probability distribution.

\begin{definition}[Reference Action Set]\label{def:ref-subset}
From each set \(r_{x,\ell}\), we pick a single representative \(a_{x,\ell}\). The resulting set $\mathcal{A}(x):=\{a_{x,\ell}\}_{\ell\leq m}$ forms a reference action set.
\end{definition}

Intuitively, the reference action set compactly represents the action space.
For each state $x$, $\mathcal{A}(x)$ captures,  the ``most representative'' actions. This reduction dramatically decreases the number of actions that need to be considered.



\subsubsection{Cluster Regret}
Finally, we define the central variable in our analysis.
Fix a state \(x\) and consider the reference action set \(\mathcal{A}(x) := \{a_{x,\ell}\}_{\ell \leq m}\). For each \(\ell \leq m\), define a random process \((Z_{x,a})_{a \in r_{x,\ell}}\) where
\[
Z_{x,a} := Q^*(x,a) - Q^*(x,a_{x,\ell}).
\]
We then define the nonnegative random variable, which we call the cluster regret:
\begin{equation}
Y_{x,\ell} := \max_{a \in r_{x,\ell}} Z_{x,a} = \max_{a \in r_{x,\ell}} Q^*(x,a) - Q^*(x,a_{x,\ell}).
\label{equ:Y-define}
\end{equation}

The cluster regret \(Y_{x,\ell}\) is central because it directly connects to the regret analysis of the reference action set as well as to the output of Algorithm~\ref{alg:valuefunction}.
To see the former, let \(\mathcal{A}\) be an arbitrary reference action set. By the definition of the partitions \(\{\Theta_{\ell}\}_{\ell \leq m}\), whose union forms the support of the environment distribution, for any sampled environment \(\theta\) there exists an \(\ell \leq m\) such that \(\theta \in \Theta_{\ell}\).
Then, by the definition of regret in Equation~\eqref{equ:regret-x-define}, we have
\begin{equation}
\begin{split}
\regret(x) &= \max_{a \in \mathcal{A}_{\full}(x)} Q^*(x,a) - \max_{b \in \mathcal{A}(x)} Q^*(x,b) \\[1ex]
&= \max_{a \in r_{x,\ell}} Q^*(x,a) - \max_{b \in \mathcal{A}(x)} Q^*(x,b) \\[0.5ex]
&\le \max_{a \in r_{x,\ell}} Q^*(x,a) - Q^*(x,a_{x,\ell}) = Y_{x,\ell}.
\end{split}
\label{equ:regret2Y}
\end{equation}
Here, the second equality uses the definition of the action space partition in Equation~\eqref{equ:cluster-define}, the inequality holds because \(a_{x,\ell} \in \mathcal{A}(x)\), and the final equality follows from the definition of the cluster regret in Equation~\eqref{equ:Y-define}.
Therefore, if restricting to the reference action set $\mathcal{A}$, the cluster regret \(Y_{x,\ell}\) upper bounds the regret incurred at state $x$ for environments in cluster \(\Theta_\ell\). 
By bounding the expected maximal cluster regret
\begin{equation} \mathbb{E}_{\theta\sim p}\left[\max_{x\in\mathcal{X},\ell\leq m} Y_{x,\ell}\right], \label{equ:max-Y-expect} \end{equation}
we effectively control the objective in Equation~\eqref{equ:obj-uniform}.
This can be summarized as the following lemma:
\begin{lemma}
For any reference action set $\mathcal{A}$, the expected maximum state regret is upper bounded by the expected maximum cluster regret:
\(
\mathbb{E}_{\theta \sim p}\Big[ \max_{x\in\mathcal{X}} \regret(x) \Big]\leq \mathbb{E}_{\theta\sim p}\left[\max_{x\in\mathcal{X},\ell\leq m} Y_{x,\ell}\right].
\)
\end{lemma}

\section{Bounds of Reference Action Set}
\label{sec:reference}

As motivated in the previous section, the expected maximal cluster regret defined in Equation~\eqref{equ:max-Y-expect} is an important quantity: it provides an upper bound on the expected maximum regret of the reference action set and is also connected to the expected maximum regret of the algorithm's output (as in the upcoming section).
We now analyze this quantity for an arbitrary reference action set.


To remind the reader, in our framework (Equation~\eqref{equ:mu-x-define}), each environment is parameterized by $\theta \sim p$, with its optimal state–action value function defined via the function $f$.
We first derive a bound in the simple case where $p$ is a Gaussian distribution and $f$ is linear.
Under this assumption, the random process $\{Q^*(x,a)\}_{(x,a)\in\mathcal{X}\times\mathcal{A}_{\full}}$ forms a centered Gaussian process. We then extend the analysis to the more general setting of non-centered sub-Gaussian processes.
Our bound will depend on the \textbf{Gaussian width} of a set $\mathcal{S}\subset\mathbb{R}^n$, a key geometric parameter defined as
\begin{equation}
G(\mathcal{S}):=\mathbb{E}\max_{s\in\mathcal{S}}\langle u, s\rangle,\; u\sim \mathcal{N}(0,I).
\label{equ:Gaussianwidth-define}
\end{equation}
Intuitively, it measures the “size” or complexity of the set $\mathcal{S}$ through the lens of Gaussian fluctuations: the larger the set, the larger the expected supremum of inner products with a standard Gaussian vector. This notion plays a central role in high-dimensional probability \citep{vershynin2018high} and asymptotic convex geometry \citep{rothvoss2021asymptotic}.
Some useful properties of Gaussian width, which will be employed later in our analysis, are provided in the Appendix~\ref{app:gaussian-width}.

\subsection{Centered Gaussian Process}

Let the environment \(\theta\) be sampled from a zero-mean Gaussian distribution, and the function \(f\) be linear, i.e.,
\(
Q^*(x,a) = \langle \theta, \phi(x,a) \rangle + c.
\)
Since the intercept $c\in\mathbb{R}$ does not affect the expected maximal cluster regret, we may simply assume \(c=0\).
Under this assumption, the random process \(\{Q^*(x,a)\}_{(x,a)\in\mathcal{X}\times\mathcal{A}_{\full}}\)is a centered Gaussian process: for any finite subset \(\mathcal{S} \subseteq \mathcal{X}\times\mathcal{A}_{\full}\), the random vector \(\{Q^*(x,a)\}_{(x,a)\in\mathcal{S}}\) is jointly Gaussian, with \(\mathbb{E}_{\theta}[Q^*(x,a)] = 0\) for all \((x,a)\).

Moreover, any such Gaussian process admits the following canonical representation \citep{vershynin2018high}, so without loss of generality we may focus on
\begin{equation}
Q^*(x,a)=\langle\theta, \phi(x,a)\rangle,\quad\forall (x,a)\in\mathcal{X}\times\mathcal{A}_{\full},\;\textnormal{ where }\theta\sim\mathcal{N}(0,I).
\label{equ:canonical-meta-MDP}
\end{equation}
This canonical form allows us to derive explicit bounds on the expected maximal cluster regret, as stated in the following theorem.

\begin{theorem}\label{thm:geometric-bounds} 
If the random process $\{Q^*(x,a)\}_{(x,a)\in\mathcal{X}\times\mathcal{A}_{\full}}$ satisfies Equation~\eqref{equ:canonical-meta-MDP}, then, for some constant \( C > 0 \), it holds that
\begin{equation*}
\mathbb{E}_{\theta}\left[\max_{x\in\mathcal{X},\ell\leq m}Y_{x,\ell} \right] \leq \max_{x\in\mathcal{X}, \ell\leq m} G(r_{x,\ell}) + C\epsilon\sqrt{\log m|\mathcal{X}|}.
\end{equation*}
\begin{equation*}
\mathbb{E}_{\theta}\left[\max_{x\in\mathcal{X}}\min_{\ell\leq m}Y_{x,\ell} \right] \geq \max_{x\in\mathcal{X}}\min_{\ell\leq m}G(r_{x,\ell}) - C\epsilon\sqrt{\log m|\mathcal{X}|}.
\end{equation*}
\end{theorem}
Here, $\epsilon$ is the largest diameter among all clusters.
Since the process \(\{Q^*(x,a)\}_{(x,a)\in\mathcal{X}\times\mathcal{A}_{\full}}\)is centered, the term $G(r_{x,\ell})$ is equivalent to the expected cluster regret $\mathbb{E}_{\theta}\left[Y_{x,\ell}\right]$.
The proof, which we will explain in detail below, is based on Proposition~2.10.8 of \cite{talagrand2014upper}.

\begin{lemma}[Borell-TIS inequality]\label{lem:Borell-TIS}
Given a set $\mathcal{S}$, and a zero-mean Gaussian process $\left(X_a\right)_{a\in\mathcal{S}}$. Let $\epsilon:=\sup_{a\in\mathcal{S}}\left(\mathbb{E} X_a^2\right)^{\frac{1}{2}}$. Then for $u>0$, we have 
\begin{equation*}
\Pr\left[\left|\sup_{a\in\mathcal{S}} X_a - \mathbb{E}\sup_{a\in\mathcal{S}} X_a \right|\geq u\right]\leq 2\exp{\left(-\frac{u^2}{2\epsilon^2}\right)}.
\end{equation*}
\end{lemma}
It means that the size of the fluctuations of $\mathbb{E}\sup_{a\in\mathcal{S}} X_a $ is governed by the size of the individual random variables $X_a$. 
Recalling Equation~\eqref{equ:L2norm}, for a canonical Gaussian process, the fluctuations within a cluster are related to the cluster diameter.
As a result, the expected maximal cluster regret can be bounded in terms of the largest cluster diameter.


\begin{lemma}[Expectation integral identity]\label{lem:integral-expectation}
Given a nonnegative random variables $X$. If $\Pr[X\geq u]\leq C\exp{\left(-\frac{u^2}{\epsilon^2}\right)}$ for any $u>0$, then $\mathbb{E}X\leq C^{\dag}\epsilon\sqrt{\log C}$,  where $C,C^{\dag},\epsilon$ are positive constants.
\end{lemma}

\begin{proof}[Proof of Theorem~\ref{thm:geometric-bounds}]
Following the reasoning in Equation~\eqref{equ:L2norm}, we have
\[\mathbb{E} Z_{x,a}^2 = \mathbb{E}(Q^*(x,a)-Q^*(x,a_{x,\ell}))^2=\|\phi(x,a)-\phi(x,a_{x,\ell})\|^2_2\leq \epsilon^2.\]
Using Lemma~\ref{lem:Borell-TIS} on the process $\{Z_{x,a}\}_{a\in r_{x,\ell}}$, we have
\begin{equation*}
\Pr\left[\left|Y_{x,\ell} - \mathbb{E}Y_{x,\ell} \right|\geq u\right]\leq 2\exp{\left(-\frac{u^2}{2\epsilon^2}\right)}\quad\forall x\in\mathcal{X},\ell\leq m.
\end{equation*}
By union bound, we have
\begin{equation*}
\Pr\left[\max_{x\in\mathcal{X},\ell\leq m}\left|Y_{x,\ell} - \mathbb{E}Y_{x,\ell} \right|\geq u\right]\leq 2m|\mathcal{X}|\exp{\left(-\frac{u^2}{2\epsilon^2}\right)}.
\end{equation*}
Using Lemma~\ref{lem:integral-expectation}, there exists an absolute constant $C>0$
\begin{equation}
\mathbb{E}\max_{x\in\mathcal{X},\ell\leq m}\left|Y_{x,\ell} - \mathbb{E}Y_{x,\ell} \right| \leq C\epsilon\sqrt{\log m|\mathcal{X}|}.\label{equ:max-deviation}
\end{equation}
Further, we have 
\begin{equation}
\begin{split}
\mathbb{E}Y_{x,\ell}=&\mathbb{E}\sup_{a\in r_{x,\ell}} Q^*(x,a) - \mathbb{E}Q^*(x,a_{x,\ell}) \\
=& \mathbb{E}\sup_{a\in r_{x,\ell}} Q^*(x,a) = G(r_{x,\ell}),
\end{split}
\label{equ:EY}
\end{equation}
where the first equality follows from the definition of $Y_{x,\ell}$.
The second uses the fact that the random process is centered, i.e., $\mathbb{E}[Q^*(x,a)] = 0$.
The last equality follows from the definition of Gaussian width (Equation~\eqref{equ:Gaussianwidth-define}) and the fact that the random process is canonical.

Finally, combining the inequality $Y_{x,\ell}\leq \mathbb{E}Y_{x,\ell}+\left|Y_{x,\ell} - \mathbb{E}Y_{x,\ell} \right|$
with Equations~\eqref{equ:max-deviation} and \eqref{equ:EY}, we obtain the following bounds:

\noindent\textbf{Upper bound:} 

\begin{equation*}
\mathbb{E}\max_{x\in\mathcal{X},\ell\leq m}Y_{x,\ell}
\leq  \max_{x\in\mathcal{X},\ell \leq m}\mathbb{E}Y_{x,\ell}+\mathbb{E}\max_{x\in\mathcal{X},\ell\leq m}\left|Y_{x,\ell} - \mathbb{E}Y_{x,\ell} \right|
\leq  \max_{x\in\mathcal{X},\ell\leq m} G(r_{x,\ell}) + C\epsilon\sqrt{\log m|\mathcal{X}|}. 
\end{equation*}

\noindent\textbf{Lower bound:} 

\begin{equation*}
\mathbb{E}\max_{x\in\mathcal{X}}\min_{\ell\leq m}Y_{x,\ell}
\geq\max_{x\in\mathcal{X}}\min_{\ell\leq m}\mathbb{E}Y_{x,\ell}-\mathbb{E}\max_{x\in\mathcal{X},\ell\leq m}\left|Y_{x,\ell} - \mathbb{E}Y_{x,\ell} \right|
\geq\max_{x\in\mathcal{X}}\min_{\ell\leq m}G(r_{x,\ell}) - C\epsilon\sqrt{\log m|\mathcal{X}|}.
\end{equation*}
\end{proof}

\begin{remark}
The canonical representation in Equation~\eqref{equ:canonical-meta-MDP} can in fact express any Gaussian process, as we explain below.
Let $\phi_i(x,a)$ denote the $i^{\text{th}}$ entry of the feature vector corresponding to the pair $(x,a)$.
This framework also extends to the case $n = +\infty$, under the additional assumption that $\sum_{i \geq 1} \phi_i(x,a)^2 < \infty$ for any pair $(x,a)$.
Consider the general setting
\begin{equation*}
Q^*(x,a) = \langle \vartheta, \phi(x,a) \rangle,\quad \forall (x,a)\in\mathcal{X}\times\mathcal{A}_{\full},\;\textnormal{ where }\vartheta \sim \mathcal{N}(0,\tau),
\end{equation*}
where $\tau$ is a positive semi-definite matrix.
Let $\theta \in \mathbb{R}^n$ follow a multivariate normal distribution with zero mean and identity covariance matrix $I$.
Then the distribution of $\vartheta$ is equivalent to that of $\tau^{1/2}\theta$.
Let $\tau_j$ denote the $j$-th row of the matrix $\tau^{1/2}$.
We have \(\vartheta = (\langle \theta, \tau_j\rangle)_{j \leq n}\). Then
\[
\langle \vartheta, \phi(x,a)\rangle
= \sum_{j \leq n} \langle \theta,\tau_j\rangle \, \phi_j(x,a)
= \Big\langle \theta, \sum_{j \leq n} \tau_j \, \phi_j(x,a) \Big\rangle
= \langle \theta, \tau^{1/2} \phi(x,a) \rangle,
\]
showing that this setup is equivalent to the canonical Gaussian case, but on a transformed feature space given by \(\{\tau^{1/2} \phi(x,a)\}_{(x,a)\in\mathcal{X}\times\mathcal{A}_{\full}}\).
Consequently, the upper bound in Theorem~\ref{thm:geometric-bounds} scales with the operator norm of \(\tau^{1/2}\), giving
\[
\mathbb{E}_{\theta}\Big[\max_{x\in\mathcal{X},\ell\leq m}Y_{x,\ell} \Big] 
\le \|\tau\|^{1/2} \max_{x\in\mathcal{X}, \ell\leq m} G(r_{x,\ell}) 
+ C \|\tau\|^{1/2} \epsilon \sqrt{\log m |\mathcal{X}|}.
\]
Here, we used Property~5 in Appendix~\ref{app:gaussian-width}. 
Intuitively, \(\|\tau\|^{1/2}\) acts as a scaling factor that stretches the Gaussian fluctuations according to the covariance structure of the features.
\end{remark}

\subsection{Non-Centered Sub-Gaussian Process}

Now we consider the case where the random process \(\{Q^*(x,a)\}_{(x,a)\in\mathcal{X}\times\mathcal{A}_{\full}}\) is non-centered and sub-Gaussian. Here, the function 
$f$ need not be linear, but the process satisfies
\begin{equation}
\Pr[|Q^*(x,a)-Q^*(x,a')|>u]\leq L\exp\left(-\frac{u^2}{2\|\phi(x,a)-\phi(x,a')\|_2^2}\right)\; \forall a,a'\in\mathcal{A}_{\full}(x),x\in\mathcal{X},
\label{equ:random-process-define}
\end{equation}
for some $L\geq 1$.
We impose no assumptions on the expectations $\mathbb{E}_{\theta}[Q^*(x,a)]$, which may vary across state–action pairs, nor on the relationships among different states.

\begin{theorem}\label{thm:geometric-bounds-sub}
If $\{Q^*(x,a)\}_{(x,a)\in \mathcal{X}\times\mathcal{A}_{\full}}$ is a random process satisfying Equation~\eqref{equ:random-process-define}, then, for some absolute constants \( C',C''> 0 \), it holds that
\begin{equation*}
\mathbb{E}_{\theta}\left[\max_{x\in\mathcal{X},\ell\leq m}Y_{x,\ell} \right] \leq C'\max_{x\in\mathcal{X}, \ell\leq m}G(r_{x,\ell}) + C''\epsilon\sqrt{\log Lm|\mathcal{X}|}.
\end{equation*}
\end{theorem}

The difference between Theorem~\ref{thm:geometric-bounds} and Theorem~\ref{thm:geometric-bounds-sub} lies only in the constants, despite the latter assuming a significantly more general setting (sub-Gaussian rather than Gaussian processes, without the centeredness assumption).
The following lemma, a minor adaptation of a result from \cite{maurer2016chain}, is essential to the proof of Theorem~\ref{thm:geometric-bounds-sub}. We defer its proof to the appendix, as it closely mirrors the analysis in the cited work.

\begin{lemma}\label{lem:random-process-twosides}
Consider the random process ${Q^*(s)}_{s\in\mathcal{S}}$ satisfying for all $s,s'\in\mathcal{S}$ and some $L\geq 1$:
\[
\Pr[|Q^*(s)-Q^*(s')|>u]\leq L\exp\left(-\frac{u^2}{2\|\phi(s)-\phi(s)\|_2^2}\right).
\]
For any $s_0\in\mathcal{S}$, there are absolute constants $C,C'>0$ that
\[
\Pr\left[\sup_{s\in\mathcal{S}} |Q^*(k)-Q^*(k_0)|>  C G(\mathcal{S}) + C'\diam(\mathcal{S}) \sqrt{\ln\frac{2L}{\delta}}\right]\leq 2\delta.
\] 
\end{lemma}

The above lemma in fact provides a deviation bound for the cluster regret. To apply it, we need the following simple lemma:

\begin{lemma}\label{lem:integral}
Let $X$ be a nonnegative random variable such that for some constants $a,b,L>0$ and $\delta \in (0,1]$,
\begin{equation*}
\Pr\!\left[\,X>a+b\sqrt{\ln\frac{2L}{\delta}}\,\right]\le \delta. 
\end{equation*}
Then there exists an absolute constant $C>0$ such that
\begin{equation*} 
\mathbb{E}[X]\leq a+Cb\sqrt{\ln(L)}.
\end{equation*}
\end{lemma}

\begin{proof}[Proof of Theorem~\ref{thm:geometric-bounds-sub}]
Using Lemma~\ref{lem:random-process-twosides} on the process $Y_{x,\ell}=\sup_{a\in r_{x,\ell}} Q^*(x,a)-Q^*(x,a_{x,\ell})$, we know that there exits absolute constants $C,C'>0$ such that
\[
\Pr\left[Y_{x,\ell}>C G(r_{x,\ell}) + C'\diam(r_{x,\ell}) \sqrt{\ln\frac{2L}{\delta}}\right]\leq \delta\; \forall x\in\mathcal{X}, \ell\leq m.
\]
Write
$$
a:=\max_{x\in\mathcal{X},\ell\le m} CG(r_{x,\ell}),\qquad b:=\max_{x\in\mathcal{X},\ell\le m} C'\diam(r_{x,\ell})=C'\epsilon.
$$
For any $\delta\in(0,1]$, the above deviation bound gives
$$
\Pr\left[Y_{x,\ell} > a + b\sqrt{\ln\frac{2L}{\delta}}\right]\leq\delta
\quad \forall x\in\mathcal{X},\ell\leq m,
$$
so by the union bound
$$
\Pr\left[\max_{x\in\mathcal{X},\ell\le m}Y_{x,\ell} > a+ b\sqrt{\ln\frac{2L}{\delta}}\right]\le m|\mathcal{X}|\delta.
$$
Set $\varepsilon:=m|\mathcal{X}|\delta$ (equivalently $\delta=\varepsilon/(m|\mathcal{X}|)$). Then for any $\varepsilon\in(0,1]$,
$$
\Pr\left[\max_{x\in\mathcal{X},\ell\le m} Y_{x,\ell} > a + b\sqrt{\ln\frac{2Lm|\mathcal{X}|}{\varepsilon}}\right]\le\varepsilon.
$$
Now apply Lemma~\ref{lem:integral} to the nonnegative r.v. $X:=\max_{x\in\mathcal{X},\ell\le m}Y_{x,\ell}$ gives
$$
\mathbb{E}\left[\max_{x\in\mathcal{X},\ell\le m}Y_{x,\ell}\right]\leq a + C''b \sqrt{\ln{Lm|\mathcal{X}|}}=\max_{x\in\mathcal{X},\ell\le m} CG(r_{x,\ell}) + C''C'\epsilon\sqrt{\ln{Lm|\mathcal{X}|}}.
$$
\end{proof}

\section{Bounds on Algorithm Output}
\label{sec:theory}
The second type of subset we analyze is the output of Algorithm~\ref{alg:valuefunction}, which is simple to implement in practice as it avoids the combinatorial explosion. We show that the expected performance of this output is close to that of a uniform coverage of the full action space (i.e., a reference action set). Specifically, we provide a general upper bound on the algorithm’s performance, which connects to the expected maximal cluster regret defined in Equation~\eqref{equ:max-Y-expect}.

Since the output of Algorithm~\ref{alg:valuefunction} is random, the expectation analyzed in this section is taken with respect to both the algorithm's randomness (i.e., the sampled $\mathcal{A}$) and the distribution over $\theta$.

\begin{theorem}\label{thm:alg-bound-upper}
Let $\mathcal{A}$ be the output of Algorithm~\ref{alg:valuefunction}. 

\begin{equation*}
\begin{split}
\mathbb{E}_{\theta,\mathcal{A}}\left[ \max_{x\in\mathcal{X}}\regret(x) \right] \leq 
\mathbb{E}_{\theta}\left[\max_{x\in\mathcal{X},\ell\leq m} Y_{x,\ell}\right] + \left(\left(\sum_{\ell\leq m}p_{\ell}(1-p_{\ell})^{2K}\right)\cdot \mathbb{E}_{\theta}\left[\left(\max_{(x,a) \in \mathcal{Z}} Q^*(x,a)\right)^2\right]\right)^{1/2},
\end{split}
\end{equation*}
where \(\mathcal{Z} := \{\, (x, a - a') \mid x \in \mathcal{X},\ a, a' \in \mathcal{A}_{\full}(x) \,\}
\).
\end{theorem}
As will become clear in the proof of this theorem, the expected maximum state regret of the algorithm’s output is upper bounded by the regret of a reference action set plus an additional term
\[
\left(\sum_{\ell\leq m}p_{\ell}(1-p_{\ell})^{2K}\right)^{1/2}\cdot\left( \mathbb{E}\left[\left(\max_{(x,a)\in\mathcal{Z}}Q^*(x,a)\right)^2\right]\right)^{1/2},
\]
which can be interpreted as the ``price'' paid for not using a universal coverage of the action space (i.e., a reference action set).
Note that the second factor, $\mathbb{E}_{\theta}\left[(\max_{(x,a) \in \mathcal{Z}}Q^*(x,a))^2\right]$, is simply a constant since the set of state–action pairs $\mathcal{Z}$ is fixed.
The quantity of real interest is the first factor, $\sum_{\ell\leq m}p_{\ell}(1-p_{\ell})^{2K}$ which characterizes how quickly this ``price'' decreases as the sample size $K$ increases.

To understand how the expected maximum state regret is decomposed, we define the ``good event'' $E_{\ell}$ as:
\[
E_{\ell}:=\left\{r_{x,\ell}\cap \mathcal{A}(x)\!\neq\!\emptyset,\;\forall x\in\mathcal{X}\right\},
\]
and by definition of the partitions $r_{x,\ell}$, the compliment event is $E_{\ell}^c=\left\{r_{x,\ell}\cap \mathcal{A}(x)\!=\!\emptyset,\;\forall x\in\mathcal{X}\right\}$.
We first note that the i.i.d. sampling of environments is independent of the events $E_{\ell}$ and $E_{\ell}^c$. This observation leads to the following lemma:
\begin{lemma}\label{lem:bound-probability}
Let \( \mathcal{A}\) be the output of Algorithm~\ref{alg:valuefunction}.  
It holds that for each $\ell\leq m$
\begin{equation*}
\begin{split}
\Pr[\theta\in \Theta_{\ell},E_{\ell}^c] 
= p_{\ell}(1-p_{\ell})^K, \quad 
\Pr[\theta\in \Theta_{\ell},E_{\ell}] 
\leq  p_{\ell}.
\end{split}
\end{equation*}

\begin{proof}
Since $\mathcal{A}$ is the output of Algorithm~\ref{alg:valuefunction}, the event $\{\theta\in \Theta_{\ell}\}$ is independent from $E$ or $E_{\ell}^c$.
We have
\begin{align*}
\Pr[\theta\in \Theta_{\ell},E_{\ell}^c]=\Pr[\theta\in \Theta_{\ell}]\Pr[E_{\ell}^c]=p_{\ell}(1-p_{\ell})^K,
\end{align*}
where the first equality uses independence, the second equality uses the definition of $p_{\ell}$ in Equation~\eqref{equ:q_l}, and the probability of missing cluster $r$ in $K$ i.i.d. samplings. Similarly, 
\begin{align*}
\Pr[\theta\in \Theta_{\ell},E_{\ell}]=p_{\ell}\cdot \Pr[E].
\end{align*}
Using $\Pr[E]\leq 1$, we complete the proof.
\end{proof}
\end{lemma}

\begin{proof}[Proof of Theorem~\ref{thm:alg-bound-upper}]
The key step is to connect the algorithmic output $\mathcal{A}$ to a reference action set.
If \( r_{x,\ell} \cap  \mathcal{A}(x) \neq \emptyset \), we select an element \( a_{x,\ell} \in r_{x,\ell} \cap  \mathcal{A}(x) \). 
If $r_{x,\ell}\cap \mathcal{A}(x)\!=\!\emptyset$, we choose an arbitrary point $a_{x,\ell}\in r_{x,\ell}$ as the representative.
The set \[\mathcal{A}':=\{a_{x,\ell}\}_{x\in\mathcal{X},\ell\leq m}\] forms a reference action set in Definition~\ref{def:ref-subset}.

By the definition of the partitions \(\{\Theta_{\ell}\}_{\ell \leq m}\), whose union forms the support of the problem environment distribution, for any sampled environment \(\theta\) there exists an \(\ell \leq m\) such that \(\theta \in \Theta_{\ell}\).
When the event $E$ occurs, we have
\begin{equation}
\begin{split}
\mathbb{E}\left[\max_{x\in\mathcal{X}}\regret(x) \Big| \theta\in \Theta_{\ell}, E_{\ell}\right]\leq & \mathbb{E}\left[\max_{x\in\mathcal{X}} Y_{x,\ell} \Big| \theta\in \Theta_{\ell}, E_{\ell}\right]\\
=& \mathbb{E}\left[\max_{x\in\mathcal{X}} Y_{x,\ell} \Big|\theta\in \Theta_{\ell} \right],\\
\leq & \mathbb{E}\left[\max_{x\in\mathcal{X},\ell \leq m}Y_{x,\ell}\Big|\theta\in \Theta_{\ell}\right],
\end{split}
\label{equ:alg-bound-1}
\end{equation}
where the first inequality uses $\regret(x)\leq Y_{x,\ell}$ when $E$ occurs (as argued in Equation~\eqref{equ:regret2Y}), and the equality uses the fact that $Y_{x,\ell}$ is independent of the event $E$. 
When the complement event $E_{\ell}^c$ occurs, we have
\begin{align}
\begin{split}
\mathbb{E}\left[\max_{x\in\mathcal{X}}\regret(x)\Big| \theta\in \Theta_{\ell},E_{\ell}^c\right]
\leq \mathbb{E}\left[\max_{(x,a)\in\mathcal{Z}} Q^*(x,a) \Big| \theta\in \Theta_{\ell}\right],\label{equ:alg-bound-2}
\end{split}
\end{align}
where the inequality uses $\regret(x)\leq \max_{a\in \mathcal{A}_{\full}(x)\ominus\mathcal{A}_{\full}(x)} Q^*(x,a)$, and the fact that this quantity is independent of $E_{\ell}^c$.

Combining the construction of the reference set with the regret decomposition under the events $E$ and $E_{\ell}^c$, we obtain the following bound:
\begin{align*}
\mathbb{E}\left[\max_{x\in\mathcal{X}}\regret(x)\right] = & \sum_{\ell\leq m} \Pr[\theta\in \Theta_{\ell},E_{\ell}]\cdot\mathbb{E}\left[\max_{x\in\mathcal{X}}\regret(x) \Big| \theta\in \Theta_{\ell},E_{\ell}\right]\\
& + \sum_{\ell\leq m} \Pr[\theta\in \Theta_{\ell},E_{\ell}^c ]\cdot\mathbb{E}\left[\max_{x\in\mathcal{X}}\regret(x)\Big| \theta\in \Theta_{\ell},E_{\ell}^c \right]\\
\leq & \sum_{\ell\leq m} p_{\ell}\cdot \left(\mathbb{E}\left[\max_{x\in\mathcal{X},\ell\leq m} Y_{x,\ell} \Big|\theta\in \Theta_{\ell} \right] +  (1-p_{\ell})^K\cdot \mathbb{E}\left[\max_{(x,a)\in\mathcal{Z}}  Q^*(x,a)\Big|\theta\in \Theta_{\ell}\right] \right) \\
= & \mathbb{E}\left[\max_{x\in\mathcal{X},\ell\leq m} Y_{x,\ell}\right] + \sum_{\ell\leq m} p_{\ell}\; (1-p_{\ell})^K\cdot \mathbb{E}\left[\max_{(x,a)\in\mathcal{Z}}  Q^*(x,a)\Big|\theta\in \Theta_{\ell}\right] \\
\leq & \mathbb{E}\left[\max_{x\in\mathcal{X},\ell\leq m} Y_{x,\ell}\right] + \left(\left(\sum_{\ell\leq m}p_{\ell}(1-p_{\ell})^{2K}\right)\cdot \mathbb{E}\left[\left(\max_{(x,a)\in\mathcal{Z}}Q^*(x,a)\right)^2\right]\right)^{1/2},
\end{align*}
where the first equality uses tower rule.
The first inequality uses Equations~(\ref{equ:alg-bound-1}-\ref{equ:alg-bound-2}), and Lemma~\ref{lem:bound-probability}.
The last equality uses tower rule again.
The last inequality uses Cauchy-Schwarz inequality, which states that for any two random variables $U$ and $V$, we have $|\mathbb{E}[UV]|\leq\sqrt{\mathbb{E}[U^2]\mathbb{E}[V^2]}$, and Jensen’s inequality that $(\mathbb{E}\left[\max_{(x,a)\in\mathcal{Z}}  Q^*(x,a)\mid\theta\in \Theta_{\ell}\right])^2\leq \mathbb{E}\left[(\max_{(x,a)\in\mathcal{Z}}  Q^*(x,a))^2\mid\theta\in \Theta_{\ell}\right]$
\end{proof}

\subsection{Approximate Q-Function and Relaxed Bound}

Step~6 of Algorithm~\ref{alg:valuefunction} assumes that the exact values of the optimal state–action value function $Q^*(x,a)$ are known. In practice, however, this assumption rarely holds, especially when the state space is large. Instead, one typically observes samples of state–action–next-state–reward tuples and constructs an approximate value function $\hat{Q}$ \citep{lagoudakis2003least}. This approximation introduces a discrepancy between $Q^*$ and $\hat{Q}$, which in turn introduce an approximation error.
We now investigate how this approximation motivates a relaxed bound on the expected maximum state regret.

Notice that in the proof of Theorem~\ref{thm:alg-bound-upper}, the argument holds as long as $r_{x,\ell}\cap\mathcal{A}\neq\emptyset$ when $\theta\in \Theta_{\ell}$. This already provides a form of relaxation: the exact solution of $\argmax_{a\in \mathcal{A}_{\full}(x)} Q^*(x,a)$ is not required. Instead, it suffices to find some action $a\in r_{x,\ell}$ whenever $\theta\in \Theta_{\ell}$ for the bound to remain valid.
In general, we may assume that, instead of selecting the exact optimal action at a state, we choose an approximate action whose distance from the optimal action is at most $\delta$.
The corresponding relaxed bound is then stated in the following corollary.

\begin{corollary}
In Algorithm~\ref{alg:valuefunction}, step 6, the optimal action $a^*:=\argmax_{a\in \mathcal{A}_{\full}(x)} Q^*(x,a)$ is replaced by an approximated action $a$ satisfying
\[
\|\phi(x,a^*)-\phi(x,a)\|_2 \leq \delta \quad\forall x\in\mathcal{X}.
\]
Let $\mathcal{A}$ denote the output of this relaxed algorithm.
Let $r_{x,\ell} \oplus \delta B$ denote the Minkowski sum of the cluster $r_{x,\ell}$ with the scaled Euclidean ball $\delta B\in\mathbb{R}^n$.
If $\{Q^*(x,a)\}_{(x,a)\in \mathcal{X}\times\mathcal{A}_{\full}}$ is a random process satisfying Equation~\eqref{equ:random-process-define}, then for the same constants \( C',C''> 0 \) in Theorem~\ref{thm:geometric-bounds-sub}, we have
\begin{equation*}
\begin{split}
\mathbb{E}_{\theta,\mathcal{A}}\left[ \max_{x\in\mathcal{X}}\regret(x) \right] \leq & C'\max_{x\in\mathcal{X}, \ell\leq m}G(r_{x,\ell}\oplus\delta B) + C''(\epsilon+\delta)\sqrt{\log Lm|\mathcal{X}|} \\
&+ \left(\left(\sum_{\ell\leq m}p_{\ell}(1-p_{\ell})^{2K}\right)\cdot \mathbb{E}\left[\left(\max_{(x,a)\in\mathcal{Z}}Q^*(x,a)\right)^2\right]\right)^{1/2}.
\end{split}
\end{equation*}
\begin{proof}
The key is to consider the new clusters of action space $r'_{x,\ell}:=r_{x,\ell} \oplus \delta B$.
Then, the proof of Theorem~\ref{thm:alg-bound-upper} applies.
\end{proof}
\end{corollary}

\begin{remark}
In the finite-dimensional case, the Euclidean ball $B$ is used to quantify the approximation error. More precisely, the Gaussian width satisfies
\[
G(r_{x,\ell} \oplus \delta B) = G(r_{x,\ell}) + \delta G(B),
\]
where 
\(
G(B) = \mathbb{E}_{\theta \sim \mathcal{N}(0,I)} \|\theta\|_2 \geq \sqrt{n} \cdot \sqrt{\frac{n}{n+1}}.
\) Here, we also used Property~2 in Appendix~\ref{app:gaussian-width}.  
This gives a meaningful bound on the approximation error when $n$ is finite. 
However, in the infinite-dimensional case $n = \infty$, the Gaussian width of a Euclidean ball is infinite, so the same argument fails. To address this, we can replace the Euclidean ball by a suitable ellipsoid to control the approximation error.
\end{remark}

\subsection{State Abstraction Based on Optimal Value Function Similarity}

Although the performance bounds established in the previous subsection depend only mildly on the logarithm of the state space cardinality, computing an optimal policy for a ground MDP with a large state space $\mathcal{X}_G$ remains computationally challenging. For readers focused on complex, multi-state environments, a standard approach to this challenge is state abstraction. This family of methods constructs a reduced abstract MDP by partitioning the ground state space $\mathcal{X}_G$ into a smaller set of abstract states $\mathcal{X}$, where states within each partition are grouped by some notions of  similarity \cite{abel2016near,li2006towards}. It is important to note, however, that this subsection may be of lesser relevance to those exclusively interested in the meta-bandits setting—a special case of our framework where the single-state nature of the problem makes state abstraction moot.

Throughout this subsection, we define the abstract MDP as \( (\mathcal{X}, \mathcal{A}_{\full}, P, R, \gamma, x_0) \) whose optimal state-action value function is $Q^*$. This is in contrast to the ground MDP, which has a larger state space \(\mathcal{X}_G\). Let \(Q^*_G\) denote the optimal state-action value function for the ground MDP.
A state aggregation function \(\psi: \mathcal{X}_G \to \mathcal{X}\) maps each ground state \(x_G\) to its corresponding abstract state. We assume that the set of available actions is the same for any ground state and its associated abstract state.

We analyze state abstraction based on approximate optimal value function similarity, following \cite{abel2016near}. However, unlike their single MDP framework, our meta-MDP setting requires a similarity assumption over the entire family of MDPs. This assumption takes the form:
\begin{equation}
\Pr[|Q^*_G(x_G,a)-Q^*_G(x'_G,a')|>u]\leq L\exp\left(-\frac{u^2}{2\|\phi(x_G,a)-\phi(x'_G,a')\|_2^2}\right),
\label{equ:random-process-state-define}
\end{equation}
for some \(L\geq 1\), and for arbitrary state-action pairs \((x_G,a)\),\((x'_G,a')\) in \(\mathcal{X}_G\times\mathcal{A}_{\full}\).
The expectations \(\mathbb{E}_{\theta}[Q^*_G(x_G,a)]\) may differ across state-action pairs. The key distinction from Equation~\eqref{equ:random-process-define} is that it assumes correlation between the values of different states.

For each abstract state $x$, define the set of state-action pairs:
\[
\mathcal{Z}_x:=\left\{(x_G,a) \mid \psi(x_G)=x,a\in\mathcal{A}_{\full}\right\} \;\forall x\in\mathcal{X}.
\]
Let $\phi(\mathcal{Z}_x)$ denote the image of $\mathcal{Z}_x$ under the feature map $\phi$. 
Define the optimal actions as:
\[
a^*_G := \argmax_{a\in \mathcal{A}_{\full}(x)}  Q^*_G(x_G, a), \quad a^*_{\mathcal{A}_{\full}} := \argmax_{a\in \mathcal{A}_{\full}(x)} Q^*(x, a),\quad a^*_{\mathcal{A}} := \argmax_{a\in \mathcal{A}(x)} Q^*(x, a).
\]

\begin{theorem}\label{thm:state-abstraction}
Given a family of ground MDPs whose optimal value functions satisfy Equation~\ref{equ:random-process-state-define}.
Suppose the ground state space is partitioned into the abstract state space $\mathcal{X}$. If the space is restricted to the action subset, the expected performance loss (suboptimality) is bounded by:
\begin{align*}
\mathbb{E}\left[\max_{x_G\in\mathcal{X}_G} Q^*_G(x_G, a^*_G) - Q^*_G(x_G, a^*_{\mathcal{A}})\right]\leq & \mathbb{E}\left[\max_{x\in\mathcal{X}}\regret(x)\right]+ \frac{1}{1-\gamma}\left(C \bar{G}+C'\bar{\epsilon}\sqrt{\ln(L|\mathcal{X}|)}\right),
\end{align*}
where $\bar{G}:=\max_{x\in\mathcal{X}} G(\phi(\mathcal{Z}_x))$ and $\bar{\epsilon}:=\max_{x\in\mathcal{X}}\diam(\phi(\mathcal{Z}_x))$.
\end{theorem}
In this bound, the term $\mathbb{E}\left[\max_{x\in\mathcal{X}}\regret(x)\right]$ is analyzed as in previous sections.
The terms $\bar{G}$ and $\bar{\epsilon}$, which are related to the sets of state-action pairs $\mathcal{Z}_x$ for each abstract state $x$, capture the effect of state abstraction. A coarser state abstraction, which results in a smaller $|\mathcal{X}|$, causes both $\bar{G}$ and $\bar{\epsilon}$ to increase.
To establish the proof of Theorem~\ref{thm:state-abstraction}, we begin by introducing a measure for the distortion that the state abstraction $\psi$ introduces into the optimal Q-values.
For each abstract state $x \in \mathcal{X}$, we define its local Q-value dispersion as:
\[\nu_x:=\sup_{z,z'\in\mathcal{Z}_x} |Q^*_G(z) - Q^*_G(z')|.\]
This quantity $\nu_x$ captures the worst-case difference between the optimal Q-values of any two ground states aggregated into the same abstract state $x$. The global Q-value dispersion is then defined as the supremum over all abstract states $\nu:=\max_{x\in\mathcal{X}}\nu_x$.
The core of the proof lies in connecting the suboptimality to this dispersion $\nu$ and the maximal regret. This connection is formalized through the following supporting lemmas.
\begin{lemma}[\cite{abel2016near}]\label{lem:ground-MDP}
For an arbitrary ground state $x_G$ and its abstract state $x = \psi(x_G)$, 
\[
|Q^*_G(x_G, a) - Q^*(x, a)| \leq \frac{\nu}{1-\gamma} \;\forall a\in\mathcal{A}_{\full}(x).
\]
\end{lemma}

\begin{lemma}\label{lem:ground-MDP-suboptimality}
The suboptimality in the ground MDP is bounded by:
\[
Q^*_G(x_G, a^*_G) - Q^*_G(x_G, a^*_{\mathcal{A}}) \leq \frac{2\nu}{1-\gamma} + \regret(x).
\]
\begin{proof}
\begin{align*}
Q^*_G(x_G, a^*_G)\leq & Q^*(x, a^*_G)  + \frac{\nu}{1-\gamma}\\
\leq& Q^*(x, a^*_{\mathcal{A}_{\full}}) +   \frac{\nu}{1-\gamma}\\
=& Q^*(x, a^*_{\mathcal{A}}) + \regret(x)+   \frac{\nu}{1-\gamma}\\
\leq & Q^*_G(x_G, a^*_{\mathcal{A}}) + \regret(x)+   \frac{2\nu}{1-\gamma},
\end{align*}
where the first and last inequalities use Lemma~\ref{lem:ground-MDP}, the second inequality uses the definition of $a^*_{\mathcal{A}_{\full}}$ and the equality uses the definition of regret.
\end{proof}
\end{lemma}

Therefore, it suffices to find an upper bound for the expectation $\mathbb{E}_{\theta}[\nu]$.

\begin{proof}[Proof of Theorem~\ref{thm:state-abstraction}]

Choose an arbitrary state-action pair $z_0\in\mathcal{Z}_x$. 
By Lemma~\ref{lem:random-process-twosides}, there are absolute constants $C,C'>0$ that
\[
\Pr\left[\max_{z\in\mathcal{Z}_x} |Q^*_G(z)-Q^*_G(z_0)|>  C G(\phi(\mathcal{Z}_x)) + C'\diam(\phi(\mathcal{Z}_x)) \sqrt{\ln\frac{2L}{\delta}}\right]\leq 2\delta.
\] 
Recall the definition of $\nu_x$.
By triangle inequality, we have 
\[
\nu_x=\max_{z,z'\in\mathcal{Z}_x}|Q^*_G(z)-Q^*_G(z')|\leq 2\max_{z\in\mathcal{Z}_x} |Q^*_G(z)-Q^*_G(z_0)|.
\]
Then
\[
\Pr\left[\frac{\nu_x}{2}>  C G(\phi(\mathcal{Z}_x)) + C'\diam(\phi(\mathcal{Z}_x)) \sqrt{\ln\frac{2L}{\delta}}\right]\leq 2\delta.
\]
By union bound and $\nu=\max_{x\in\mathcal{X}}\nu_x$
\[
\Pr\left[\frac{\nu}{2}>  C \bar{G} + C'\bar{\epsilon}\sqrt{\ln\frac{2L}{\delta}}\right]\leq 2|\mathcal{X}|\delta.
\]
By Lemma~\ref{lem:integral}, there exists an absolute constant $C''>0$ such that
\begin{equation*} 
\mathbb{E}\left[\nu\right]\leq 2C \bar{G}+2C'C''\bar{\epsilon}\sqrt{\ln(2L|\mathcal{X}|)}.
\end{equation*}
By Lemma~\ref{lem:ground-MDP-suboptimality}, the expected suboptimality is bounded by 
\begin{align*}
\mathbb{E}\left[\max_{x_G\in\mathcal{X}_G} Q^*_G(x_G, a^*_G) - Q^*_G(x_G, a^*_{\mathcal{A}})\right]\leq & \frac{2}{1-\gamma}\mathbb{E}[\nu] + \mathbb{E}\left[\max_{x\in\mathcal{X}}\regret(x)\right]
\end{align*}
Plug in the upper bound of $\mathbb{E}\left[\nu\right]$, we complete the proof.
\end{proof}

\section{Examples}
\label{sec:examples}

To ground our theoretical discussion, we examine three concrete examples of increasing complexity. We begin with a meta-bandit problem with i.i.d. actions to demonstrate that our framework, while specifically designed for correlated action settings, naturally reduces to and handles the independent case. We then progress to structured meta-MDPs consisting of two distinct tabular MDPs. In both of these initial examples, each state-action pair $(x,a)$ corresponds to a unit vector $\mathbf{1}_{(x,a)}$, providing a clear, interpretable structure for analyzing representations.
We conclude with practical cartpole environments where the physical parameters (like pole length and mass) form natural groupings, or ``clusters,'' creating a more realistic and challenging meta-RL benchmark where action correlations emerge from the underlying physics.
All code for reproducing our experiments can be found at \url{https://github.com/Quan-Zhou/Meta-RL-Representative-Action-Selection}.

\subsection{Meta-Bandit with IID Actions}
\label{app:orthogonal-basis}

Consider a meta-bandit example where $\mathcal{X}$ is a singleton. We can then drop the subscript $x$, and the objective reduces to Equation~\eqref{equ:regret-define}.
Let the action space be the orthonormal basis of $\mathbb{R}^n$, i.e.,
\(
\mathcal{A}_{\full} = \{ \mathbf{1}_i : i = 1, \dots, n \}.
\) 
Each action $a$ is associated with a unit vector $\mathbf{1}_{a}$ that has value $1$ at coordinate $a$ and $0$ elsewhere. We define the canonical Gaussian process as
\[
Q^*(a):= \langle \mathbf{1}_{a}, \theta \rangle ,\quad \theta\sim\mathcal{N}(0,I).
\]
Under this construction, the collection $\{Q^*(a)\}_{a \in \mathcal{A}_{\full}}$ consists of i.i.d. standard normal variables.



\paragraph{Expected Regret for Algorithm Output.}  
Since the actions are independent, for any subset \(\mathcal{A}\) of fixed cardinality, the performance \(\mathbb{E}\max_{a\in\mathcal{A}} Q^*(a)\) depends only on the size of \(\mathcal{A}\), not on its specific elements.  
Therefore, the expected regret of Algorithm~\ref{alg:valuefunction} is determined solely by the number of distinct elements in its output.

Because \(\max_{a \in \mathcal{A}_{\full}} Q^*(a)\) is simply the maximum among the \(n\) coordinates of \(\theta\), and by symmetry each coordinate has the same probability of achieving the maximum, the algorithm effectively reduces to sampling with replacement. The probability that \(|\mathcal{A}| = N\) is thus the probability of obtaining \(N\) distinct actions in \(K\) independent draws. By the tower rule, the exact expected regret can be computed as
\begin{equation}
\begin{split}
\mathbb{E}_{\theta,\mathcal{A}}[\regret]
&=\sum_{N\leq\min\{n,K\}} \Pr\left[|\mathcal{A}| = N\right]\;\mathbb{E}_{\theta,\mathcal{A}}[\regret\mid |\mathcal{A}| = N]\\
&= \sum_{N\leq\min\{n,K\}} \frac{\binom{n}{N} \sum_{i=0}^{N} (-1)^i \binom{N}{i} (N-i)^K}{n^K} \; \left(\mathbb{E} \max_{i=1,\dots,n} \theta_i -\mathbb{E} \max_{i=1,\dots,N} \vartheta_i  \right),
\end{split}
\label{equ:regret-iid}
\end{equation}
where each \(\theta_i\) and \(\vartheta_i\) is an i.i.d.\ \(\mathcal{N}(0,1)\) sample.

\paragraph{Upper Bounds for Algorithm Output.}  
Theorem~\ref{thm:alg-bound-upper} provides an upper bound that depends on a partition of the action set into clusters. Consider partitioning the \(n\) unit vectors into \(m\) clusters, each containing \(n/m\) actions (assuming \(n/m\) is an integer).  
By construction, any reference action set based on this partition has the same expected regret, which can be computed as
\[
\mathbb{E}_{\theta}[\regret] 
= \mathbb{E} \max_{i=1,\dots,n} \theta_i - \mathbb{E} \max_{i=1,\dots,n/m} \vartheta_i,
\]
for any reference action set. Since each cluster has the same probability of containing the optimal action, \(p_\ell = 1/m\) for \(\ell \le m\).

Furthermore, the additional term in Theorem~\ref{thm:alg-bound-upper} simplifies because \(\{Q^*(a)\}_{a \in \mathcal{A}_{\full}}\) is a canonical Gaussian process:
\begin{align*}
\mathbb{E}\left[\left(\max_{(x,a)\in\mathcal{A}_{\full}\ominus\mathcal{A}_{\full}}Q^*(x,a)\right)^2\right]
&\leq (G(\mathcal{A}_{\full}\ominus\mathcal{A}_{\full}))^2 + 4\diam(\mathcal{A}_{\full})\\
&= 4(G(\mathcal{A}_{\full}))^2 + 4\sqrt{2}\\
&\leq 8\log n + 4\sqrt{2},
\end{align*}
where the first inequality uses Appendix~\ref{app:squared-Gaussian-width}, the equality follows from Property~2 in Appendix~\ref{app:gaussian-width} and \(\diam(\mathcal{A}_{\full})=\sqrt{2}\), and the last inequality uses \(G(\mathcal{A}_{\full})=\mathbb{E}\|\theta\|_{\infty}=\mathbb{E} \max_{i=1,\dots,n} \theta_i\leq \sqrt{2\log n}\) (Appendix~\ref{app:bound-maxgaussian}).

Finally, as shown in the proof of Theorem~\ref{thm:alg-bound-upper}, the algorithm’s expected regret is upper-bounded by the regret of the reference action set plus this additional term:
\begin{align}
\mathbb{E}_{\theta,\mathcal{A}}[\regret] 
\leq  \left(\mathbb{E} \max_{i=1,\dots,n} \theta_i -\mathbb{E} \max_{i=1,\dots,n/m} \vartheta_i\right)
+ \left(1-\frac{1}{m}\right)^K\cdot \left(8\log n + 4\sqrt{2}\right)^{1/2}.
\label{equ:regret-iid-upper}
\end{align}

\paragraph{Numerical Comparison.}  
Figure~\ref{fig:bounds-iid} compares the exact expected regret with the upper bounds in Equation~\eqref{equ:regret-iid-upper}. We consider \(n = 10\) actions and vary the number of clusters as \(m \in \{2, 5, 10\}\). The dashed curves show the upper bounds for \(K = 1, 2, \ldots, 50\), while the solid curves show the exact expected regret computed from Equation~\eqref{equ:regret-iid}.  
For both Equations~\eqref{equ:regret-iid} and~\eqref{equ:regret-iid-upper}, the term \(\mathbb{E}\max_{i=1,\dots,n} \theta_i - \mathbb{E}\max_{i=1,\dots,L} \vartheta_i\) (for an integer \(L\)) is estimated using \(10^5\) Monte Carlo samples.

Since all three upper bounds hold simultaneously, the tightest bound at each $K$ is obtained by taking their minimum. In the early stage ($K < 10$), the bound corresponding to $m = 2$ clusters is the tightest. However, for larger sample sizes ($K > 20$), the bound with $m = 10$ clusters becomes the tightest. This figure illustrates that as $K$ increases, using finer-grained partitions yields tighter bounds.

\begin{figure}[htp]
    \centering
    \includegraphics[width=0.7\linewidth]{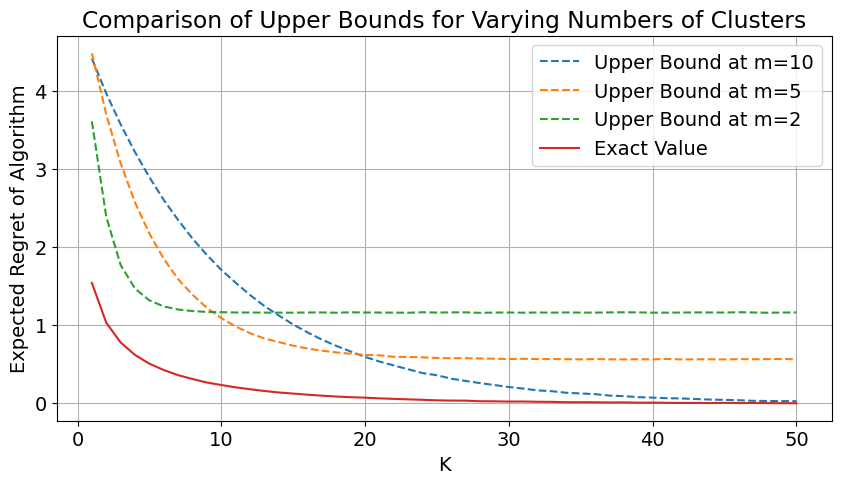}
    \caption{Comparison between the exact expected regret of Algorithm~\ref{alg:valuefunction} (the solid curve) and the theoretical upper bounds (dashed curves) from Theorem~\ref{thm:alg-bound-upper}, for $n = 10$ actions (iid from $\mathcal{N}(0,1)$) and different numbers of clusters $m \in {2, 5, 10}$. The bounds and exact values are plotted as functions of the sample sizes $K = 1, 2, \ldots, 50$.}
    \label{fig:bounds-iid}
\end{figure}

\subsection{A Family of Tabular MDPs}

In a tabular MDP, the transition probabilities $P(x' \mid x,a)$ and rewards $R(x,a)$ are stored explicitly for each $(x,a)$.
It is in fact a special case of a linear MDP in which each state-action pair $(x,a)$ is represented by a unit feature vector:
\[
\phi(x,a) = \mathbf{1}_{(x,a)} \in \mathbb{R}^{|\mathcal{X} \times \mathcal{A}_{\full}|}.
\]
We consider a family of MDP environments $\theta$ sampled from two tabular MDPs: $\theta_1$ with probability $\rho$ and $\theta_2$ with probability $1-\rho$. Each MDP has two states $\mathcal{X} = \{x_1,x_2\}$, two actions $\mathcal{A}_{\full}(x) = \{a_1,a_2\}$, and a discount factor $\gamma = 0.5$. 
In this setup, the action space corresponds to a 4-dimensional orthogonal basis.

\paragraph{The First MDP}

The transition probabilities and rewards of $\theta_1$ are summarized in Table~\ref{tab:mdp-example-1}.

\begin{table}[h!]
\centering
\begin{subtable}[t]{0.45\textwidth}
\centering
\begin{tabular}{|c|c|c|c|}
\hline
State & Action & Next State & Reward \\
\hline
$x_1$ & $a_1$ & $x_1$   & 2 \\
$x_1$ & $a_2$ & $x_2$   & 0 \\
$x_2$ & $a_1$ & $x_1$   & 1 \\
$x_2$ & $a_2$ & $x_2$   & 3 \\
\hline
\end{tabular}
\caption{First MDP ($\theta_1$).}
\label{tab:mdp-example-1}
\end{subtable}
\hfill
\begin{subtable}[t]{0.45\textwidth}
\centering
\begin{tabular}{|c|c|c|c|}
\hline
State & Action & Next State & Reward \\
\hline
$x_1$ & $a_1$ & $x_2$   & 0 \\
$x_1$ & $a_2$ & $x_1$   & 2 \\
$x_2$ & $a_1$ & $x_2$   & 3 \\
$x_2$ & $a_2$ & $x_1$   & 1 \\
\hline
\end{tabular}
\caption{Second MDP ($\theta_2$).}
\label{tab:mdp-example-2}
\end{subtable}
\caption{Transition probabilities and rewards for the two MDPs in the family.}
\end{table}

The Bellman optimality equations for the first MDP are
\[
V^*(x_1) = \max(2 + 0.5 V^*(x_1),\, 0 + 0.5 V^*(x_2)), \quad
V^*(x_2) = \max(1 + 0.5 V^*(x_1),\, 3 + 0.5 V^*(x_2)),
\]
which yield $V^*(x_1)=4$ and $V^*(x_2)=6$. The corresponding parameter vector defined in Equation~\eqref{equ:linear-Q-define} is $\theta_1 = [4,3,3,6]^\top$, so that $Q^*(x_1,a_1)=4$, $Q^*(x_1,a_2)=3$, $Q^*(x_2,a_1)=3$, and $Q^*(x_2,a_2)=6$, with optimal policy $\pi^*(x_1)=a_1$, $\pi^*(x_2)=a_2$. Under the suboptimal policy $\pi(x_1)=a_2$, $\pi(x_2)=a_1$, the value function becomes $V(x_1)=2/3$, $V(x_2)=4/3$, giving a maximum performance loss from $x_2$ of $V^*(x_2)-V(x_2)=6-4/3$.

\paragraph{The Second MDP} 
The second MDP $\theta_2$ is obtained by flipping the optimal actions in each state relative to $\theta_1$ (Table~\ref{tab:mdp-example-2}). The optimal value function remains $V^*(x_1)=4$, $V^*(x_2)=6$, with corresponding parameter vector $\theta_2=[3,4,6,3]^\top$ and optimal policy $\pi^*(x_1)=a_2$, $\pi^*(x_2)=a_1$. Under the opposite (suboptimal) policy, the value function becomes $V(x_1)=2/3$, $V(x_2)=4/3$, so the maximum performance loss from $x_2$ is unchanged: $V^*(x_2)-V(x_2)=6-4/3$.

\paragraph{Expected Performance Difference of Algorithm Output} 
Running Algorithm~\ref{alg:valuefunction} with sample size $K$ gives rise to three cases: 
\begin{itemize}
    \item Only $\theta_1$ is sampled in Algorithm~\ref{alg:valuefunction} in $K$ iterations (probability $\rho^K$), producing $\mathcal{A}(x_1)=a_1$ and $\mathcal{A}(x_2)=a_2$. A performance loss occurs if a newly drawn MDP is $\theta_2$ (probability $1-\rho$).  
    \item Only $\theta_2$ is sampled in Algorithm~\ref{alg:valuefunction} in $K$ iterations (probability $(1-\rho)^K$), producing $\mathcal{A}(x_1)=a_2$ and $\mathcal{A}(x_2)=a_1$. A performance loss occurs if a newly drawn MDP is $\theta_1$ (probability $\rho$).  
    \item Both $\theta_1$ and $\theta_2$ are sampled at least once in Algorithm~\ref{alg:valuefunction} in $K$ iterations (probability $1 - \rho^K - (1-\rho)^K$), so all actions are included, resulting in zero performance loss.
\end{itemize}

The performance loss, defined as the difference between the optimal policy achievable by the algorithm's output and the optimal policy over the full action space, can be quantified as follows. Its expectation over the family of MDP environments, starting from the initial state $x_2$ (where the loss is maximal), is
\begin{equation}
\mathbb{E}_{\theta,\mathcal{A}}\big[ V^*(x_2) - V^{\pi}(x_2) \big] 
= \left(6 - \frac{4}{3}\right) \Big( \rho^K (1-\rho) + (1-\rho)^K \rho \Big).
\label{equ:performance-tabluar}
\end{equation}

\paragraph{Upper Bound for Algorithm Output} 
For each state, the full action set consists of two actions. We set $m=2$ and treat each action as a separate cluster ($Y_{x,\ell}=0$). Then, for any reference action set, $\mathbb{E}_{\theta}[\max_{x,\ell} Y_{x,\ell}] = 0$. Furthermore, 
\begin{align*}
\mathbb{E}\Big[\Big(\max_{(x,a)\in\mathcal{Z}} Q^*(x,a)\Big)^2\Big] 
&= \rho \, \mathbb{E}\Big[\Big(\max_{(x,a)\in\mathcal{Z}} Q^*(x,a)\Big)^2 \,\Big|\, \theta_1\Big] 
+ (1-\rho) \, \mathbb{E}\Big[\Big(\max_{(x,a)\in\mathcal{Z}} Q^*(x,a)\Big)^2 \,\Big|\, \theta_2\Big] \\
&= \rho \,(Q^*(x_2,a_2)-Q^*(x_2,a_1))^2 + (1-\rho) \,(Q^*(x_2,a_1)-Q^*(x_2,a_2))^2 = 9.
\end{align*}
Hence, the upper bound in Theorem~\ref{thm:alg-bound-upper} reduces to
\begin{equation}
\mathbb{E}_{\theta \sim p}\Big[ \max_{x\in\mathcal{X}} \regret(x) \Big] 
\le 3 \, \big( \rho (1-\rho)^{2K} + \rho^{2K} (1-\rho) \big)^{1/2}.
\label{equ:bound-tabluar}
\end{equation}

\paragraph{Numerical Comparison} 
Figure~\ref{fig:tabluar} compares the exact expected performance loss, given by Equation~\eqref{equ:performance-tabluar}, with the upper bound from Theorem~\ref{thm:alg-bound-upper} (Equation~\eqref{equ:bound-tabluar}). By Equation~\eqref{equ:performance-bound-connect}, these quantities satisfy
\[
\text{Expected Performance Difference} \le \frac{1}{1-\gamma} \mathbb{E}_{\theta \sim p}\Big[ \max_{x\in\mathcal{X}} \regret(x) \Big].
\]
We compute both sides for $K=1,\dots,20$ and $\rho \in [0.01,0.99]$ and plot them in Figure~\ref{fig:tabluar}, where the left plot corresponds to the left-hand side of the inequality and the right plot to the right-hand side; both plots share the same $x$, $y$, and $z$ limits for direct comparison.  

Figure~\ref{fig:tabluar} illustrates how the bound and performance loss evolve with the concentration parameter $\rho$. When either $\rho$ or $1-\rho$ is very small, the expected performance loss decreases rapidly in the first few draws but approaches zero more slowly. Conversely, when $\rho$ and $1-\rho$ are nearly uniform, the decrease is slower initially but reaches zero more quickly.

\begin{figure}[htp]
    \centering
    \includegraphics[width=0.9\linewidth]{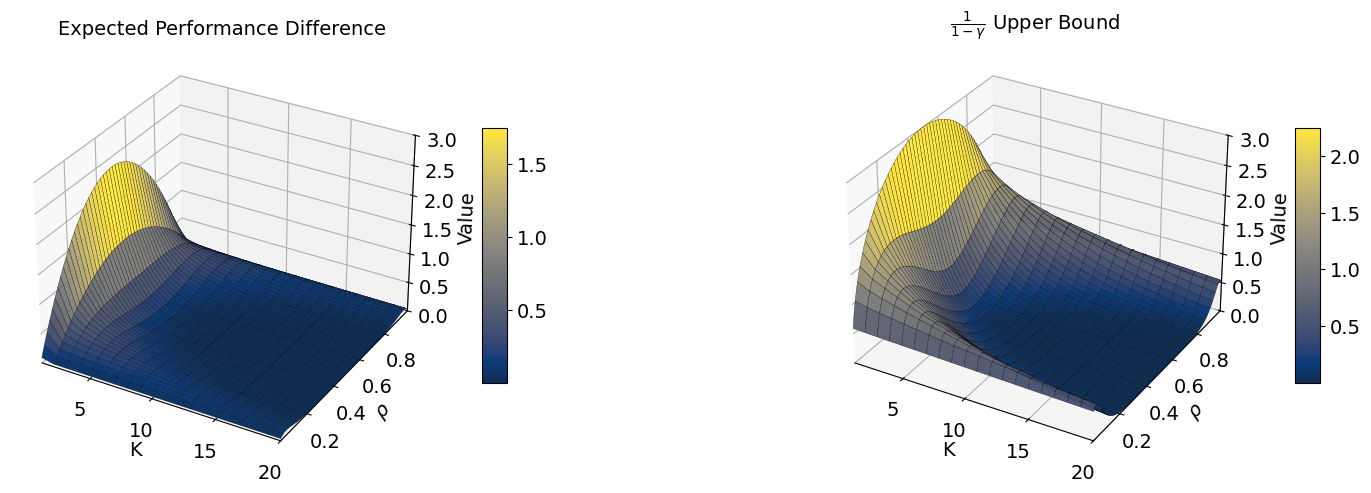}
    \caption{Comparison of the exact expected performance loss (left) with the upper bound from Theorem~\ref{thm:alg-bound-upper} (right) for different sample sizes $K$ and concentration parameters $\rho$. The figure illustrates how the performance loss and bound evolve as $\rho$ varies: when either $\rho$ or $1-\rho$ is small, the loss decreases quickly initially but approaches zero more slowly, whereas for nearly uniform $\rho$, the decrease is slower at first but reaches zero faster.}
    \label{fig:tabluar}
\end{figure}

\subsection{Case Study: The CartPole Environment}

The CartPole environment \citep{towers2024gymnasium} is a standard benchmark for control algorithms. In the continuous variant studied here, the agent applies continuous force to a cart moving along a track to balance an attached pole. The state space comprises the cart's position and velocity, along with the pole's angle and angular velocity, with the objective of maintaining stability within defined bounds. The agent receives a reward of +1 for every time step the pole remains balanced, with episodes terminating when the pole falls or position limits are exceeded.

We employ a parameterizable environment generator that creates environments with randomized physical properties across two difficulty categories. The 'easy' configuration features lower gravity (8.0), lighter pole mass (0.08), and shorter length (0.4) with ±5\% variation, while the 'medium' configuration uses standard gravity (9.8), heavier mass (0.1), longer pole (0.5) with ±10\% variation. For each experiment, environments are generated by sampling parameters uniformly within these variation ranges.

To enable tabular reinforcement learning, we discretize the continuous 4D state space into 125 discrete states. This discretization ignores cart position (first dimension) and applies non-uniform binning to the remaining three dimensions, concentrating resolution around critical central values. The continuous action space (force applied to cart) is discretized into 501 actions uniformly spaced between -100 and 100.

Action sets are generated using Algorithm~\ref{alg:valuefunction} for varying values of $K$. The training process employs Q-learning over 10000 episodes, with episodes terminating either when the pole falls or when the temporal difference error remains below $10^{-3}$ for five total steps (not necessarily consecutive), using a discount factor $\gamma = 0.99$.

Figure~\ref{fig:cartpole} reports means and standard deviations for both training runtime and total reward across 30 repetitions for $K = [3,5,8,10,12,15,20]$, comparing performance when the agent uses the full action space (blue) versus the compressed action set (orange). The training runtime measures the time required to train a policy in a newly sampled environment using the given available actions, following the same Q-learning procedure used during action set generation (10000 episodes, $\gamma=0.99$, with termination when the pole falls or TD error drops below $10^{-3}$ for five total steps). The total reward represents the cumulative reward obtained by deploying the learned policy from the fixed initial state $[0,0,0,0]$ for a single episode, as both the environment and starting state are deterministic once sampled.

As shown in Figure~\ref{fig:cartpole}, the experimental results align with theoretical expectations. Since the full action space remains constant across different values of $K$, the blue curves for the action space remain horizontal in both plots, serving as a consistent baseline.
In the runtime analysis (left plot), the orange curve for the action set demonstrates a gradual increase in training time as $K$ grows. This reflects the computational cost of expanding the action set, though even at larger $K$ values, the action set maintains a significant efficiency advantage over the full action space.
The performance comparison (right plot) reveals that both the action space and action set achieve nearly identical cumulative rewards across all $K$ values. This indicates that our action sets effectively capture the essential control capabilities of the full action space, despite their substantially reduced size.
Notably, the action set curves exhibit much tighter error bars in both plots, demonstrating significantly more stable and predictable performance in both training efficiency and policy effectiveness. This enhanced stability underscores the practical advantage of using carefully constructed action sets over the full action space for reinforcement learning in continuous control tasks.

\begin{figure}
\centering
\includegraphics[width=\linewidth]{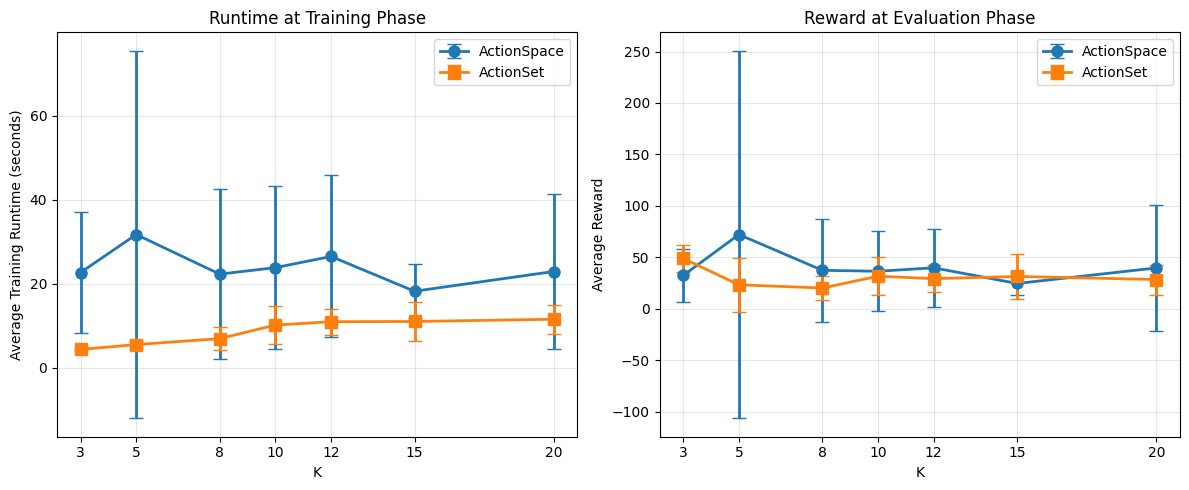}
\caption{Performance comparison between full action space (501 actions) and action sets (constructed via Algorithm~\ref{alg:valuefunction}) on a family of CartPole environments.
Left: Training runtime for policies using available actions.
Right: Total reward earned by deployed policies from fixed initial state. Results show mean $\pm$ one standard deviation across 30 repetitions for $K = [3,5,8,10,12,15,20]$.}
\label{fig:cartpole}
\end{figure}

\section*{Acknowledgements}

We extend our sincere gratitude to colleagues who enriched this work: Dr. Mark Kozdoba for incisive questions that shaped our theoretical framework; Dr. Zhou Feng for pivotal discussions on epsilon net construction that advanced our algorithm; and Mengjia Niu, Jikai Yan for meticulous review and constructive feedback on the writing.

\section*{}
\bibliographystyle{abbrvnat}
\bibliography{ref}

\appendix

\section{Proofs}

For completeness, we include here the proofs of the lemmas used in the main body of the paper.

\subsection{Proof of Lemma~\ref{lem:integral-expectation}}
\begin{proof}
By the tail-integral formula for a nonnegative random variable $X$,
\begin{align*}
\mathbb{E} [X] 
=& \int_0^{\infty}\Pr[X\geq u]\; du \\
=& \int_0^{u_0}\Pr[X\geq u]\; du + \int_{u_0}^{\infty}\Pr[X\geq u]\; du\\
\leq & u_0 + \frac{1}{u_0} \int_{u_0}^{\infty}u\cdot\Pr[X\geq u]\; du\\
\leq & u_0 + \frac{C}{u_0} \int_{u_0}^{\infty}u\cdot \exp{\left(-\frac{u^2}{\epsilon^2}\right)}\; du\\
= & u_0 + \exp\left(-\frac{u_0^2}{\epsilon^2}\right)\cdot\frac{c\epsilon^2}{2u_0}\\
= & \epsilon\sqrt{\log C} + \frac{\epsilon}{2\sqrt{\log C}} \leq C^{\dag} \epsilon\sqrt{\log C},
\end{align*}
where the first equality uses integrated tail formula of expectation (cf. Lemma~1.6.1 of \cite{vershynin2018high}). The last equality set $u_0:=\epsilon\sqrt{\log C}$.
\end{proof}

\subsection{Proof of Lemma~\ref{lem:integral}}
\begin{proof}
By the tail-integral formula for a nonnegative random variable,
$$
\mathbb{E}[X]=\int_{0}^{\infty}\Pr[X>t]\,dt
\le a+\int_{0}^{\infty}\Pr(X>a+s)\,ds.
$$
For any fixed $s\ge0$, set $\delta_s:=\min\{1,\,2L e^{-(s/b)^2}\}$.
Then $b\sqrt{\ln\frac{2L}{\delta_s}}\le s$, so by monotonicity of the tail,
$$
\Pr[X>a+s]\le \delta_s=\min\{1,\,2L e^{-(s/b)^2}\}.
$$
Hence,
$$
\mathbb{E}[X]\le a+\int_{0}^{\infty}\min\{1,\,2L e^{-(s/b)^2}\}\,ds.
$$
Split the integral at $s_0:=b\sqrt{\ln(2L)}$:
$$
\mathbb{E}[X]\le a+\underbrace{\int_{0}^{s_0}1\,ds}_{=\,b\sqrt{\ln(2L)}}
+\underbrace{\int_{s_0}^{\infty}2L e^{-(s/b)^2}\,ds}_{(\dagger)}.
$$
For $(\dagger)$, substitute $u=s/b$:
$$
(\dagger)=2Lb\int_{\sqrt{\ln(2L)}}^{\infty}e^{-u^2}\,du.
$$
Using the inequality (shift-and-bound),
$$
\int_{\alpha}^{\infty} e^{-u^2}\,du
=\int_{0}^{\infty}e^{-(x+\alpha)^2}\,dx
\le e^{-\alpha^2}\int_{0}^{\infty}e^{-x^2}\,dx
= e^{-\alpha^2}\,\frac{\sqrt{\pi}}{2},
$$
with $\alpha=\sqrt{\ln(2L)}$, we obtain
$$
(\dagger)\le 2Lb\cdot e^{-\ln(2L)}\cdot \frac{\sqrt{\pi}}{2}
= b\cdot \frac{\sqrt{\pi}}{2}.
$$
Putting everything together,
$$
\mathbb{E}[X]\le a+b\Big(\sqrt{\ln(2L)}+\frac{\sqrt{\pi}}{2}\Big)\leq a+Cb\sqrt{\ln(L)}.
$$
\end{proof}

\subsection{Proof of Lemma~\ref{lem:random-process-twosides}}


To simplify the notation, consider the random process $\{Q(s)\}_{s\in\mathcal{S}}$ satisfying for all $s,s'\in\mathcal{S}$ and some $L\geq 1$:
\begin{equation}
\Pr[|Q(s)-Q(s')|>u]\leq L\exp\left(-\frac{u^2}{2\|\phi(s)-\phi(s)\|_2^2}\right).
\label{equ:majorizing-ass}
\end{equation}

The proof use Talagrand's generic chaining and majorizing measure \cite{talagrand1992simple}:

\begin{definition}\label{def:chaining}
Fix $r \geq 2$ and let $k_0$ be the largest integer that $2r^{-k_0} \geq \mathrm{diam}(\mathcal{S})$. For $k \geq k_0$, take finite $T_k \subset \mathcal{S}$ and maps $\tau_k: \mathcal{S} \to T_k$ and the corresponding inverse $\tau^{\dag}_k: T_k \to \mathcal{S}$ that maps each $t$ to its preimage:
\[
\tau^{\dag}_k(t) = \{s \in \mathcal{S} : \tau_k(s) = t\}.
\] 
We assume they satisfy:
\begin{enumerate}
    \item $\tau_k(t) = t$ for $t \in T_k$
    \item $\tau_k(s) = \tau_k(s') \Rightarrow \tau_{k-1}(s) = \tau_{k-1}(s')$
    \item $\mathrm{diam}(\tau_k^{\dag}(t)) \leq 2r^{-k}$ for $t \in T_k$.
\end{enumerate}

Fix $s_0 \in \mathcal{S}$ and set $\tau_{k_0}(s) = s_0$ for all $s$. Then for any $s\in\mathcal{S}$, we have the \textbf{chaining} identity:
\begin{equation}
Q(s)-Q(s_0) = \sum_{k>k_0} \left(Q(\tau_k(s))-Q(\tau_{k-1}(s)) \right),\label{equ:chaining}
\end{equation}

For $k\geq k_0$, set the partition $\mathcal{T}_k = \{\tau_k^{\dag}(t)\}_{t\in T_k}$, which form an increasing sequence with $\mathcal{T}_{k_0} = \{\mathcal{S}\}$.
\end{definition}

\begin{lemma}[\cite{talagrand1992simple}]\label{lem:majorizing}
There is an universal constant $C$ such that for every finite $\mathcal{S}\subset\mathbb{R}^n$ there is an increasing sequence of partitions $\mathcal{T}_k$ of $\mathcal{S}$ (as defined in Definition~\ref{def:chaining}) and a probability measure $\mu$ on $\mathcal{S}$, such that for all $t\in T_k$,  
\[
\sup_{s\in\mathcal{S}} \sum_{k>k_0}^{\infty} r^{-k}\sqrt{\log\frac{1}{\mu(\tau^{\dag}_k\circ\tau_k (s))}} \leq C G(\mathcal{S}).
\]
\end{lemma}

Let $\mu$ and $\mathcal{T}_k,T_k$ be as determined for $\mathcal{S}$ by Definition~\ref{def:chaining} and Lemma~\ref{lem:majorizing}.
For $k\geq k_0$, we define a function $\xi_k:\mathcal{T}_k\to\mathbb{R}$ as follows:
\[
\xi_k(A):=r^{-k+1}\sqrt{8\log\left(\frac{2^{k-k_0}L}{\mu(A)\delta}\right)}.
\]
\begin{claim}[\cite{maurer2016chain}]\label{lem:majorizing-aux1}
Let $s_0$ be an arbitrary point in $\mathcal{S}$.
For any $\delta\in(0,1)$
\begin{equation*}
\Pr\left[\exists s\in\mathcal{S}: Q(s)-Q(s_0)\geq \sum_{k>k_0} \xi_k(\tau^{\dag}_k\circ\tau_k(s))\right]<\delta.
\end{equation*}
\end{claim}
\begin{proof}[Proof of Claim~\ref{lem:majorizing-aux1}]

\begin{align*}
\Pr\left[\exists s: Q(s)-Q(s_0)> \sum_{k>k_0} \xi_k(\tau^{\dag}_k\circ\tau_k (s))\right] = & \Pr\left[\exists s: \sum_{k>k_0} \left(Q(\tau_k(s))-Q(\tau_{k-1}(s))-\xi_k(\tau^{\dag}_k\circ\tau_k (s)) \right)> 0\right]\\
\leq & \Pr\left[\exists s,k>k_0: Q(\tau_k(s))-Q(\tau_{k-1}(s))-\xi_k(\tau^{\dag}_k\circ\tau_k (s))>  0\right]\\
\leq & \Pr\left[\exists k>k_0, t\in T_k: Q(\tau_k(t))-Q(\tau_{k-1}(t))>\xi_k(\tau_k^{\dag}(t))\right]\\
\leq & \sum_{k>k_0}\sum_{t\in T_k} \Pr\left[Q(\tau_k(t))-Q(\tau_{k-1}(t))>\xi_k(\tau_k^{\dag}(t)) \right]\\
\leq & \sum_{k>k_0}\sum_{t\in T_k} L \exp\left(-\frac{\xi_k(\tau_k^{\dag}(t))^2}{2\|\tau_k(t)-\tau_{k-1}(t)\|_2^2}\right) \\
\leq & \sum_{k>k_0}\sum_{t\in T_k} L \exp\left(-\frac{\xi_k(\tau_k^{\dag}(t))^2}{2(2r^{-k+1})^2}\right)\\
= & \delta\sum_{k>k_0}\frac{1}{2^{k-k_0}}\sum_{t\in T_k}\mu(\xi_k(\tau_k^{\dag}(t)))\\
= & \delta\sum_{k>k_0}\frac{1}{2^{k-k_0}}=\delta.
\end{align*}
Here, the first equation uses the chaining in Equation~\eqref{equ:chaining}. The first inequality uses that if the sum is positive, at least one of the term has to be positive. The second inequality make equivalent change of notation.
The third inequality uses the union bound.
The fourth inequality uses the sub-Gaussian assumption in Equation~\eqref{equ:majorizing-ass}.
The last inequality uses that for $t\in T_k$ both $\tau_k(t)$ and $\tau_{k-1}(t)$ are members of $\tau^{\dag}_{k-1}(t)$, we must have $\|\tau_k(t)-\tau_{k-1}(t)\|_2\leq 2r^{-k+1}$.
The second equality uses the definition of $\xi_k$ function.
The last equality uses the $\mu$ is a probability distribution (cf. Lemma~\ref{lem:majorizing}).
\end{proof}

Following the same procedures, we can prove another claim (because the chaining identity still holds, and that the deviation bound in Equation~\eqref{equ:majorizing-ass} holds for both sides.)
\begin{claim}\label{lem:majorizing-aux2}
Let $s_0$ be an arbitrary point in $\mathcal{S}$.
For any $\delta\in(0,1)$
\begin{equation*}
\Pr\left[\exists s\in\mathcal{S}: Q(s_0)-Q(s)\geq \sum_{k>k_0} \xi_k(\tau^{\dag}_k\circ\tau_k (s))\right]<\delta.
\end{equation*}
\end{claim}

The main proof continues here using these claims.

\begin{proof}[Proof of Lemma~\ref{lem:random-process-twosides}]
Using Claim~\ref{lem:majorizing-aux1}\&\ref{lem:majorizing-aux2}, and the union bound, we have
\begin{align*}
\Pr\left[\sup_{s\in\mathcal{S}} |Q(k)-Q(k_0)|> \sum_{k>k_0} \xi_k(\tau^{\dag}_k\circ\tau_k (s))\right] \leq \Pr\left[\exists s: |Q(k)-Q(k_0)|> \sum_{k>k_0} \xi_k(\tau^{\dag}_k\circ\tau_k (s))\right]\leq 2\delta.
\end{align*}
Further, note that
\begin{align*}
\sum_{k>k_0} \xi_k(\tau^{\dag}_k\circ\tau_k (s)) =& r\sum_{k>k_0} r^{-k}\sqrt{8\log\left(\frac{2^{k-k_0}L}{\mu(\tau^{\dag}_k\circ\tau_k (s))\delta}\right)} \\
\leq & r\sum_{k>k_0} r^{-k}\sqrt{8\log\left(\frac{2^{k-k_0}L}{\delta}\right)} + r\sum_{k>k_0} r^{-k}\sqrt{8\log\left(\frac{1}{\mu(\tau^{\dag}_k\circ\tau_k (s))}\right)} \\
= & r^{-k_0+1}\sum_{k>0} r^{-k}\sqrt{8\log\left(\frac{2^{k}L}{\delta}\right)} + r\sum_{k>k_0} r^{-k}\sqrt{8\log\left(\frac{1}{\mu(\tau^{\dag}_k\circ\tau_k (s))}\right)}\\
\leq & \sqrt{8}r^{-k_0+1}\sum_{k>0} \sqrt{k}r^{-k} \sqrt{\log\left(\frac{2L}{\delta}\right)} +\sqrt{8}rC G(\mathcal{S})\\
\leq & \sqrt{8}(r^2/2)\sum_{k>0} \sqrt{k}r^{-k+1} \diam(\mathcal{S})\sqrt{\log\left(\frac{2L}{\delta}\right)} + \sqrt{8}rC G(\mathcal{S}).
\end{align*}
Here, the first equality uses the definition of $\xi_k$ function. 
The first inequality uses $\sqrt{a+b}\leq \sqrt{a}+\sqrt{b}$ for $a,b\geq 0$.
The second inequality uses Lemma~\ref{lem:majorizing}.
The last inequality uses that by the definition of $k_0$ (in Definition~\ref{def:chaining}), $2r^{-k_0-1}<\diam(\mathcal{S})$, such that $r^{-k_0+1}<(r^2/2)\diam(\mathcal{S})$.
Let $C'=\sqrt{8}rC$ and $C''=\sqrt{8}(r^2/2)\sum_{k>0} \sqrt{k}r^{-k+1}$.
Since $s_0$ is an arbitrary point, we have shown that for any $s_0\in\mathcal{S}$,
\[
\Pr\left[\sup_{s\in\mathcal{S}} |Q(k)-Q(k_0)|>  C G(\mathcal{S}) + C'\diam(\mathcal{S}) \sqrt{\ln\frac{2L}{\delta}}\right]\leq 2\delta.
\] 

\end{proof}

\section{Properties of Gaussian width}
\label{app:gaussian-width}
\noindent
\paragraph{Property 1: $G(\mathcal{S})=G(\ominus\mathcal{S})$.}
\begin{align*}
G(\mathcal{S})=\mathbb{E}\left[\max_{a\in\mathcal{S}}\langle a,-\theta\rangle\right]=\mathbb{E}\left[\max_{a\in\mathcal{S}}\langle -a,\theta\rangle\right] =\mathbb{E}\left[\max_{a'\in\ominus\mathcal{S}}\langle a',\theta\rangle\right]=G(\ominus\mathcal{S}),
\end{align*}
where the first equality uses $-\theta$ and $\theta$ are identically distributed. The third equality uses for any $a\in\mathcal{S}$, it holds that $-a\in\ominus\mathcal{S}$.
\paragraph{Property 2: $G(\mathcal{S}_1\oplus\mathcal{S}_2)=G(\mathcal{S}_1)+G(\mathcal{S}_2)$.} Let $a^*(\mathcal{S},\theta):=\argmax_{a\in\mathcal{S}}\langle a,\theta\rangle$ denote the optimal action in $\mathcal{S}$ given environment $\theta$.
By definition, we have
\[
\langle a^*(\mathcal{S}_1\oplus\mathcal{S}_2,\theta),\theta\rangle \geq \langle a+a',\theta\rangle\quad\forall a\in\mathcal{S}_1,a'\in\mathcal{S}_2. 
\]
Since $a^*(\mathcal{S}_1,\theta)\in\mathcal{S}_1$ and $a^*(\mathcal{S}_2,\theta)\in\mathcal{S}_2$, we have
\[
\langle a^*(\mathcal{S}_1\oplus\mathcal{S}_2,\theta),\theta\rangle \geq \langle a^*(\mathcal{S}_1,\theta) + a^*(\mathcal{S}_2,\theta),\theta\rangle. 
\]

Also, since $a^*(\mathcal{S}_1\oplus\mathcal{S}_2,\theta)\in\mathcal{S}_1\oplus\mathcal{S}_2$, there exists $a_1\in\mathcal{S}_1$ and $a_2\in\mathcal{S}_2$ such that $a^*(\mathcal{S}_1\oplus\mathcal{S}_2,\theta)=a_1+a_2$. By definition, we have 
\[
\langle a^*(\mathcal{S}_1,\theta),\theta\rangle \geq \langle a_1,\theta\rangle,\quad\langle a^*(\mathcal{S}_2,\theta),\theta\rangle \geq \langle a_2,\theta\rangle,
\]
such that
\[
\langle a^*(\mathcal{S}_1\oplus\mathcal{S}_2,\theta),\theta\rangle = \langle a_1+a_2,\theta\rangle\leq \langle a^*(\mathcal{S}_1,\theta),\theta\rangle + \langle a^*(\mathcal{S}_2,\theta),\theta\rangle.
\]

Combining both, we have 
\[
\langle a^*(\mathcal{S}_1\oplus\mathcal{S}_2,\theta),\theta\rangle = \langle a^*(\mathcal{S}_1,\theta) + \langle a^*(\mathcal{S}_2,\theta),\theta\rangle.
\]

Therefore, we have
\begin{align*}
G(\mathcal{S}_1\oplus\mathcal{S}_2)=&\mathbb{E}\left[\langle a^*(\mathcal{S}_1\oplus\mathcal{S}_2,\theta),\theta\rangle\right]
= \mathbb{E}\left[\langle a^*(\mathcal{S}_1,\theta)+a^*(\mathcal{S}_2,\theta),\theta\rangle\right]
= G(\mathcal{S}_1)+G(\mathcal{S}_2).
\end{align*}

\paragraph{Property 3: $G(\mathcal{S})\leq \frac{\sqrt{n}}{2}\diam(\mathcal{S})$.}
\begin{align*}
G(\mathcal{S})= &\frac{1}{2}\left(G(\mathcal{S})+G(\ominus\mathcal{S}) \right) \\
= &\frac{1}{2}\left(G(\mathcal{S}\ominus\mathcal{S}) \right) \\
\leq& \frac{1}{2}\mathbb{E}\max_{a,a'\in\mathcal{S}}\|\theta\|_2\|a-a'\|_2 \\
\leq&\frac{1}{2}\mathbb{E}\diam(\mathcal{S}) \|\theta\|_2 \leq \frac{\sqrt{n}}{2}\diam(\mathcal{S}),
\end{align*}
where the first equality follows from Property 1, and the second from Property 2.
The first inequality uses the Cauchy–Schwarz inequality, the second uses the definition of the diameter, and the last uses $\mathbb{E}\|\theta\|_2\leq \sqrt{n}$.

\paragraph{Property 4: $G(\mathcal{S})\geq\frac{1}{\sqrt{2\pi}}\diam(\mathcal{S})$.}
\begin{align*}
G(\mathcal{S})=& \frac{1}{2}G(\mathcal{S}\ominus\mathcal{S})\\
=& \frac{1}{2} \mathbb{E}\max_{a,b\in\mathcal{S}}\langle a-b,\theta\rangle\\
\geq & \frac{1}{2} \mathbb{E}\max\left\{\langle a-b,\theta\rangle,\langle b-a,\theta\rangle\right\}\\
= & \frac{1}{2} \mathbb{E}\max\left|\langle a-b,\theta\rangle\right|\\
= & \frac{1}{2}\sqrt{\frac{2}{\pi}}\|a-b\|_2,
\end{align*}
where the first equality uses Property~1\&2. The inequality fix arbitrary $a,b\in\mathcal{S}$ and uses that $a-b,b-a\in\mathcal{S}\ominus\mathcal{S}$. The last equality uses that $\langle a-b,\theta\rangle\sim\mathcal{N}(0,\|a-b\|_2^2)$, and that $E|\theta|=\sqrt{\frac{2}{\pi}}$ for $\theta\sim\mathcal{N}(0,1)$.
Taking the supremum over all $a,b\in\mathcal{S}$ gives the result.

\paragraph{Property 5: $G(A\mathcal{S})\leq \|A\|G(\mathcal{S})$, where $A\in\mathbb{R}^{n\times m}$.}

To prove this, we need the well-known Sudakov-Fernique's inequality:
\begin{lemma}[Sudakov-Fernique's inequality]\label{lem:Sudakov-Fernique}
Let $\{X_s\}_{s\in\mathcal{S}}$ and $\{Y_s\}_{s\in\mathcal{S}}$ be two mean zero Gaussian processes. If for all $t,s\in\mathcal{S}$, we have 
\(
\mathbb{E}(X_t-X_s)^2\leq \mathbb{E}(Y_t-Y_s)^2.
\)
Then,
\(
\mathbb{E}\sup_{s\in\mathcal{S}}X_s\leq \mathbb{E}\sup_{t\in\mathcal{S}}Y_t.
\)
\end{lemma}

Let $\{X_s\}_{s\in\mathcal{S}}$ be a canonical Gaussian process.
We consider two processes $\{X_{As}\}_{s\in\mathcal{S}}$ and $\{\|A\|X_{s}\}_{s\in\mathcal{S}}$. Note that by the definition of operator norm $\|A\|$, we have
\[
\left(\mathbb{E}(X_{At}-X_{As})^2\right)^{1/2}=\|At-As\|_2 \leq \|A\|\|t-s\|_2=\left(\mathbb{E}(\|A\|X_t-\|A\|X_s)^2\right)^{1/2}.
\]
By Lemma~\ref{lem:Sudakov-Fernique}, we have that $G(A\mathcal{S})\leq \|A\|G(\mathcal{S})$.

\section{Squared Version of Gaussian Width}
\label{app:squared-Gaussian-width}

Let \(\theta \sim \mathcal{N}(0, I)\). Given a set \(\mathcal{S}\), the \textbf{squared version of the Gaussian width} satisfies
\[
\mathbb{E}\left[\left(\max_{a\in\mathcal{S}\ominus\mathcal{S}}\langle a, \theta\rangle\right)^2\right]
\leq \bigl(G(\mathcal{S}\ominus\mathcal{S})\bigr)^2 + 4\,\diam(\mathcal{S}).
\]

\begin{proof}
Let $X:=\max_{a\in\mathcal{S}}\langle a, \theta\rangle-\mathbb{E}\max_{a\in\mathcal{S}}\langle a, \theta\rangle$.
By Lemma~\ref{lem:Borell-TIS}:
Then for $u>0$ and $\varepsilon^2:=\sup_{a\in\mathcal{S}}\mathbb{E}(\langle a, \theta\rangle)^2$ we have 
\begin{equation*}
\Pr\left[\left|X \right|\geq u\right]\leq 2\exp{\left(-\frac{u^2}{2\epsilon^2}\right)}.
\end{equation*}
Then, uses integrated tail formula of expectation
\begin{align}
\mathbb{E}|X|^2=&\int^{\infty}_{0} \Pr\left[\left|X \right|^2\geq u\right] du=\int^{\infty}_{0} \Pr\left[\left|X \right|\geq u\right] 2t\;dt \leq 4\int^{\infty}_{0} t\exp{\left(-\frac{t^2}{2\varepsilon^2}\right)} dt = 4 \varepsilon^2.
\label{equ:p-norm-bound}
\end{align}

Therefore, 
\begin{align*}
\mathbb{E}\left[\left(\max_{a\in\mathcal{S}}\langle a, \theta\rangle\right)^2\right] =& \left(\mathbb{E}\left[\max_{a\in\mathcal{S}}\langle a, \theta\rangle\right]\right)^2 + \text{Var}\left(\max_{a\in\mathcal{S}}\langle a, \theta\rangle\right) \\
=& (G(\mathcal{S}))^2+ \text{Var}\left(\max_{a\in\mathcal{S}}\langle a, \theta\rangle\right)\\
\leq & (G(\mathcal{S}))^2 + 4\sup_{a\in\mathcal{S}}\mathbb{E}(\langle a, \theta\rangle)^2,
\end{align*}
where the first equality uses \(\mathbb{E}[U^2] = (\mathbb{E}[U])^2 + \text{Var}(U)\). The second uses the definition of Gaussian width. The inequality uses the identity $\text{Var}\left(\max_{a\in\mathcal{S}}\langle a, \theta\rangle\right) = \mathbb{E}\left[|X|^2 \right]$ and Equation~\eqref{equ:p-norm-bound}.
Note that for the canonical Gaussian process, $\mathbb{E}(\langle a-a', \theta\rangle)^2=\|a-a'\|_2$, such that 
\[
\sup_{a\in\mathcal{S}\ominus\mathcal{S}}\mathbb{E}(\langle a, \theta\rangle)^2 = \sup_{a\in\mathcal{S},a'\in\mathcal{S}}\mathbb{E}(\langle a-a', \theta\rangle)^2 = \sup_{a\in\mathcal{S},a'\in\mathcal{S}}\|a-a'\|_2\leq \diam(\mathcal{S}).
\]
\end{proof}

\section{Bounds of Expected Maximum of Gaussian}
\label{app:bound-maxgaussian}
Let $X_1,\dots,X_N$ be $N$ random Gaussian variables (no necessarily independent) with zero mean and variance of marginals smaller than $\tau^2$, then
\begin{equation*}
\mathbb{E}\left[\max_{i=1,\dots,N} X_i\right]\leq \tau\sqrt{2\log N}.
\end{equation*}
\begin{proof}
for any $\delta>0$,
\begin{align*}
\mathbb{E}\left[\max_{i=1,\dots,N} X_i\right] =& \frac{1}{\delta} \mathbb{E}\left[\log \exp(\delta\max_{i=1,\dots,N} X_i)\right]
\leq \frac{1}{\delta}\log \mathbb{E}\left[ \exp(\delta\max_{i=1,\dots,N} X_i)\right]\\
= & \frac{1}{\delta}\log \mathbb{E}\left[ \max_{i=1,\dots,N}  \exp(\delta X_i)\right]
\leq  \frac{1}{\delta}\log \sum_{i=1}^N \mathbb{E}\left[\exp(\delta X_i)\right]\\
\leq & \frac{1}{\delta}\log \sum_{i=1}^N \exp(\tau^2\delta^2/2) = \frac{\log N}{\delta}+\frac{\tau^2\delta}{2},
\end{align*}
where the first inequality uses Jensen's inequality. Taking $\delta:=\sqrt{2(\log N)/\tau^2}$ yields the results. 
\end{proof}

\end{document}